\documentclass{article}

\usepackage[preprint]{neurips_2026}
\usepackage[utf8]{inputenc}
\usepackage[T1]{fontenc}
\usepackage{microtype}
\usepackage{graphicx}
\usepackage[caption=false,font=small,labelfont=bf]{subfig}
\usepackage{booktabs}
\usepackage{float}
\usepackage{wrapfig}
\usepackage{placeins}
\usepackage{needspace}
\usepackage{amsmath,amsfonts,amssymb,mathtools,amsthm}
\usepackage{xcolor}
\usepackage{hyperref}
\usepackage{url}
\hypersetup{
  pdftitle={Prediction--Loss Alignment for Sampler-Robust Flow Matching Training},
  pdfauthor={Jiadong Hong, Lei Liu, Xinyu Bian, Wenjie Wang, Zhaoyang Zhang},
  colorlinks=true,
  linkcolor=[rgb]{0.10,0.25,0.55},
  citecolor=[rgb]{0.10,0.25,0.55},
  urlcolor=[rgb]{0.10,0.25,0.55}
}

\graphicspath{{./}{imgs/}}

\theoremstyle{plain}
\newtheorem{theorem}{Theorem}[section]
\newtheorem{proposition}[theorem]{Proposition}

\theoremstyle{definition}
\newtheorem{assumption}[theorem]{Assumption}
\theoremstyle{remark}

\newcommand{\Ex}{\mathbb{E}}
\newcommand{\R}{\mathbb{R}}
\newcommand{\norm}[1]{\left\lVert #1\right\rVert}

\title{Prediction--Loss Alignment for\\
Sampler-Robust Flow Matching Training}

\author{%
  Jiadong Hong\textsuperscript{1}\quad
  Lei Liu\textsuperscript{1}\quad
  Xinyu Bian\textsuperscript{2}\quad
  Wenjie Wang\textsuperscript{2}\quad
  Zhaoyang Zhang\textsuperscript{1}\\[3pt]
  {\normalfont\small\textsuperscript{1}Zhejiang University, Hangzhou, China\quad
  \textsuperscript{2}Huawei Technologies Co., Ltd, Hong Kong}\\
  {\normalfont\footnotesize\texttt{\{jiadong5,lei\_liu,zhzy\}@zju.edu.cn}\quad
  \texttt{\{bian.xinyu,wang.wenjie\}@huawei.com}}
}

\begin{document}

\maketitle

\begin{abstract}
Recent work has popularized a practical recipe in diffusion and flow matching: predict the clean signal $x$, convert it to a velocity, and train through a velocity-space loss. The conversion contains a singular endpoint amplification and therefore appears prone to unstable optimization, yet recent systems obtain strong empirical results with this recipe. We investigate this tension through the integrability of the pre-optimizer stochastic-gradient second moment. Under stated initialization conditions, the moment diverges under Uniform sampling; boundary-suppressing sampling can restore integrability under an additional upper-growth condition. We then show that prediction--loss alignment eliminates this conversion-induced source of non-integrability. Under a uniform moment bound, alignment yields a finite second moment for every timestep density, including Uniform sampling. Controlled experiments across continuous and binary settings reproduce the predicted sampler-dependent instability and show that aligned objectives remain trainable across the tested samplers. These results reconcile pointwise amplification with sampler-dependent empirical success and support alignment as a principled route to more robust flow-matching training.
\end{abstract}

\section{Introduction}
\label{sec:introduction}

Recent work has brought renewed attention to clean-data prediction in diffusion and flow matching. ``Just image Transformers'' (JiT) predicts clean images in pixel space, while Embedded Language Flows (ELF) adopts an analogous recipe in its denoising branch for continuous language embeddings \citep{li2025back,hu2026elf}. JiT and the denoising branch of ELF both predict the clean target $x$, convert that prediction into a velocity, and apply a velocity-space loss. We refer to this pairing as $x$-prediction+$v$-loss.

The potential issue is visible from three elementary equations. Let $X\sim p_{\mathrm{data}}$ and $E\sim\mathcal N(0,I)$. Along the linear path from noise to data,
\begin{equation}
Z_t=tX+(1-t)E,\qquad t\in(0,1),
\label{eq:path}
\end{equation}
the target velocity is $V=X-E$. If a network directly predicts the clean signal $\hat X_\theta(Z_t,t)$, its implied velocity is
\begin{equation}
\hat V_\theta(Z_t,t)=\frac{\hat X_\theta(Z_t,t)-Z_t}{1-t}.
\label{eq:implied_velocity}
\end{equation}
Consequently, writing $\ell_x=\norm{\hat X_\theta-X}^2$, the velocity loss is
\begin{equation}
\ell_v=\norm{\hat V_\theta-V}^2
=\frac{1}{(1-t)^2}\ell_x.
\label{eq:conversion_identity}
\end{equation}
Equation~\eqref{eq:conversion_identity} places every remaining signal-prediction error under an inverse-square endpoint factor. As $t\to1$, the corresponding sample gradient is amplified by $(1-t)^{-2}$ and its second moment by $(1-t)^{-4}$. On its face, this recipe appears poorly conditioned: rare samples near the clean endpoint may dominate stochastic updates. Yet JiT and ELF report strong empirical performance with precisely this prediction--loss pairing. This apparent contradiction motivates the central question of our work: \emph{why does this seemingly singular recipe work, and when does it fail?}

We investigate this question through complementary theoretical analysis and controlled experiments. We focus the analysis on the integrability of the pre-optimizer stochastic-gradient second moment over training timesteps. The pointwise amplification alone does not determine the outcome. Its integrated effect also depends on how the prediction residual and parameter gradient behave near the endpoint, and on how frequently the timestep sampler exposes that region. We formalize this interaction through a sufficient endpoint criterion for initialized networks.

This analysis provides a mechanism through which empirical success and endpoint amplification can coexist. Under the stated upper-growth conditions, boundary-suppressing samplers make the integrated gradient moment finite; under the initialization lower bounds, Uniform sampling instead gives divergence. Our controlled studies reproduce both sides of this behavior: mismatched training exhibits orders-of-magnitude pre-optimizer gradient amplification and can collapse under Uniform sampling, while boundary-suppressing sampling can recover competitive training. Thus, sampler choice can make the mismatched recipe viable without removing its underlying pointwise amplification.

We further study \emph{prediction--loss alignment}: train an $x$-predictor directly with an $x$-space loss and use its implied velocity only for ODE generation. Alignment removes the conversion factor from training. Under the stated uniform gradient-moment conditions, the aligned gradient second moment is integrable for every timestep density, including Uniform sampling. In this precise sense, alignment makes flow-matching training more robust to timestep-sampler choice for the mismatch-induced endpoint failure mode.

Our contributions are:
\begin{itemize}
    \item We locate the apparent instability of $x$-prediction+$v$-loss in a sampler-dependent gradient-moment integrability problem and derive a sufficient endpoint criterion for initialized direct-head and identity-residual networks.
    \item We explain why boundary-suppressing samplers can conceal mismatch-induced divergence, and show that prediction--loss alignment removes the mismatch-induced endpoint factor from training.
    \item Controlled continuous and binary experiments exhibit the predicted endpoint amplification together with a sampler-dependent training reversal, while aligned objectives remain trainable across the tested samplers.
\end{itemize}

Together, these results reconcile the singular conversion with sampler-dependent empirical success and support alignment as a principled route to flow-matching training that is more robust to timestep-sampler choice.

\section{Related Work}
\label{sec:related}

\paragraph{Parameterization and weighting.}
Diffusion objectives have long been understood as noise-level-weighted regression problems, with practical performance depending on preconditioning, parameterization, and loss weighting \citep{ho2020denoising,karras2022edm,kingma2023understanding,hang2023minsnr}. Kingma and Gao further interpret the noise schedule as an importance-sampling distribution for a weighted objective, separating that objective from its Monte Carlo estimator. Flow matching is conventionally formulated as velocity regression \citep{lipmanflow,liurflow}. Recent clean-data parameterizations instead predict $x$ directly and may convert that prediction into velocity space for training \citep{li2025back,hu2026elf}. Our work isolates this specific $x$-prediction+$v$-loss recipe and characterizes how its terminal pre-optimizer gradient amplification interacts with timestep sampling.

\paragraph{Timestep sampling.}
Non-uniform noise-level or timestep sampling is an established component of diffusion and rectified-flow training \citep{karras2022edm,esser2024scaling}. Logit-Normal sampling is used in the Stable Diffusion 3 rectified-flow recipe and subsequently adopted by JiT \citep{esser2024scaling,li2025back}. This practice is directly relevant to the empirical puzzle motivating our work: it reduces exposure to the clean endpoint, precisely where converting an $x$-prediction into velocity space amplifies prediction error. We characterize when this suppression makes the stochastic-gradient second moment integrable and when the underlying endpoint amplification remains exposed.

\paragraph{Discrete and analog-bit modeling.}
Discrete generative models may operate directly on discrete state spaces through discrete-time transition kernels or continuous-time transition rates \citep{austin2021d3pm,lou2024discrete,campbell2024generative}, or use continuous-state representations and continuous flows for categorical endpoints \citep{chenanalog,eijkelboom2024vfm,dieleman2022cdcd}. We use an analog-bit representation from the latter family as a controlled binary setting for testing the same optimization mechanism under continuous Gaussian interpolation. Our contribution concerns prediction--loss conversion and endpoint exposure within this continuous-state representation.

\section{Problem Formulation and Assumptions}
\label{sec:framework}

\subsection{Conditional flow-matching setup}

\begin{wrapfigure}[17]{r}{0.48\textwidth}
\centering
\includegraphics[width=\linewidth,trim=0.8bp 0.8bp 1.3bp 3.7bp,clip]{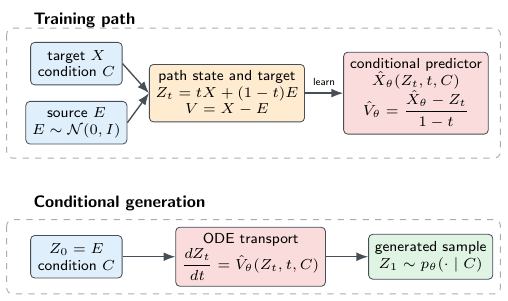}
\caption{\textbf{Conditional flow matching.} Training forms $Z_t$ and $V$; generation integrates $\hat V_\theta(Z_t,t,C)$ from $Z_0$ to $Z_1$.}
\label{fig:problem_setup}
\end{wrapfigure}

We study conditional flow matching, which learns a time-dependent velocity field that transports a tractable source distribution to a data distribution conditioned on side information. Under a linear Gaussian interpolation, we examine how clean-signal prediction, velocity-space supervision, and timestep sampling interact; Figure~\ref{fig:problem_setup} summarizes the setup.

Let $(X,C)$ denote the target data and its optional conditioning variable, with $X\in\R^D$ and $\Ex\norm{X}^2<\infty$. Let $E\sim\mathcal N(0,I)$ be independent of $(X,C)$. For a fixed $t\in(0,1)$, the linear interpolation and its samplewise velocity target are
\begin{equation}
Z_t=tX+(1-t)E,
\qquad V=X-E.
\label{eq:formal_path}
\end{equation}
Training regresses the conditional velocity along this path; generation or conditional inference integrates the learned velocity field from the Gaussian source while holding $C$ fixed. Here $t$ denotes a deterministic path location. We reserve $T$ for a random training timestep whose distribution is specified with the objective below.

\WFclear\vspace{1.5\baselineskip}
For fixed $t$, all expectations and essential suprema are over the joint law of $(X,C,E)$ and, where applicable, independent model randomness such as dropout. Predictors and conditional means may depend on $C$; we suppress this argument when it is not needed explicitly.

\subsection{Prediction--loss conversion and stochastic optimization}

\paragraph{Pointwise conversion.}
Recall Equations~\eqref{eq:path}--\eqref{eq:conversion_identity}. Fix a parameter state $\theta$, write $r_\theta(t)=\hat X_\theta(Z_t,t)-X$ and $\ell_x(t)=\norm{r_\theta(t)}^2$, and define $a(t):=(1-t)^{-2}$. Since $a(t)$ is independent of $\theta$,
\begin{equation}
\ell_v(t)=a(t)\ell_x(t),\qquad
g_x(t):=\nabla_\theta\ell_x(t),\qquad
g_v(t):=\nabla_\theta\ell_v(t)=a(t)g_x(t).
\label{eq:formal_conversion}
\end{equation}
Here the sampled input $Z_t$ is not differentiated through, and $g_x(t)$ and $g_v(t)$ are the per-sample parameter gradients before any optimizer-specific transformation.

At every fixed $t<1$, multiplying the loss by the positive scalar $a(t)$ leaves its unrestricted population minimizer unchanged. Thus, in the idealized infinite-capacity limit, the two formulations target the same conditional predictor; Appendix~\ref{app:flow_target_proof} verifies that its implied velocity recovers the fixed-time population-optimal flow-matching target. The distinction appears in finite training: one network shares parameters across all timesteps, while $a(t)$ gives errors near $t=1$ greater weight and multiplies their parameter gradients. The conversion can therefore change how finite model capacity is allocated across timesteps and alter the stochastic-gradient sequence received by the optimizer.

\paragraph{Timestep sampling.}
At each training step, we sample a data--condition pair $(X,C)$, noise $E$, and a timestep $T$ from a density $q$ on $(0,1)$, then form $Z_T=TX+(1-T)E$ and evaluate the selected loss. Equivalently, for a small interval around $t$,
\begin{equation}
T\sim q,\qquad
\Pr(T\in[t,t+\mathrm{d}t])\approx q(t)\,\mathrm{d}t.
\label{eq:timestep_sampling}
\end{equation}
Thus $q(t)$ controls how frequently training visits each timestep region.

\paragraph{Optimizer-facing second moment.}
The random parameter gradients presented to the optimizer and the conditional aligned-gradient moment are
\begin{equation}
\hat g_x=g_x(T),\qquad
\hat g_v=a(T)g_x(T),\qquad
G_\theta(t):=\Ex\norm{g_x(t)}^2.
\label{eq:optimizer_gradients}
\end{equation}
Conditional on sampling near $t$, the aligned and mismatched gradient second moments are $G_\theta(t)$ and $a(t)^2G_\theta(t)$, respectively. Their local contributions to the timestep-integrated second moment are therefore
\begin{equation}
\mathrm{d}M_x(t)=q(t)G_\theta(t)\,\mathrm{d}t,
\qquad
\mathrm{d}M_v(t)=q(t)a(t)^2G_\theta(t)\,\mathrm{d}t.
\label{eq:local_second_moment}
\end{equation}
Integrating these local contributions gives the exact total moments
\begin{equation}
M_x(q)=\Ex\norm{\hat g_x}^2=\int_0^1q(t)G_\theta(t)dt,
\qquad
M_v(q)=\Ex\norm{\hat g_v}^2=\int_0^1q(t)a(t)^2G_\theta(t)dt.
\label{eq:optimizer_second_moments}
\end{equation}
The total second moment characterizes the aggregate magnitude of raw stochastic gradients over timestep sampling. Its divergence identifies a failure of finite-second-moment regularity, with rare endpoint samples contributing increasingly large gradients. For Uniform sampling, $q(t)=1$ and mismatch inserts the factor $a(t)^2$ directly. A boundary-suppressing sampler can instead make the integral finite by reducing endpoint exposure, without removing the pointwise conversion factor.

\subsection{Initialization regimes}

Our primary question is how prediction--loss mismatch affects a flow-matching network trained from scratch. We therefore begin at the initialized parameter state and characterize the pre-optimizer gradient before training changes the prediction residual. This requires making the initialized prediction head explicit. Let $\theta_0$ denote the initialized parameters,
\begin{equation}
r_0(t):=\hat X_{\theta_0}(Z_t,t)-X,
\qquad
J_0(t):=\nabla_\theta\hat X_{\theta_0}(Z_t,t).
\label{eq:init_objects}
\end{equation}

\begin{assumption}[Initialization regimes]
\label{ass:init_regimes}
At the fixed initialized state $\theta_0$, one of the following regimes holds throughout $t\in(t_0,1)$ for some $t_0\in(0,1)$ and $c_J>0$:
\begin{description}
    \item[Regime A (direct head).] The predictor has no end-to-end identity skip and
    $\Ex\norm{J_0(t)^\top r_0(t)}^2\ge c_J$.
    \item[Regime B (identity residual).] The residual branch is zero-initialized, so
    $\hat X_{\theta_0}(Z_t,t)=Z_t$, and
    $\Ex\norm{J_0(t)^\top r_0(t)/(1-t)}^2\ge c_J$.
\end{description}
\end{assumption}

The two regimes imply the following lower bounds on the initialized aligned-gradient moment:
\begin{equation}
G_{\theta_0}(t)\ge
\begin{cases}
4c_J, & \text{Regime A},\\
4c_J(1-t)^2, & \text{Regime B}.
\end{cases}
\label{eq:init_orders}
\end{equation}

\paragraph{Proof sketch.}
Since $G_{\theta_0}(t)=4\Ex\norm{J_0(t)^\top r_0(t)}^2$, Regime A follows immediately. In Regime B,
\begin{align}
r_0(t)
&=Z_t-X=tX+(1-t)E-X=(1-t)(E-X),\\
G_{\theta_0}(t)
&=4(1-t)^2\Ex\norm{J_0(t)^\top(E-X)}^2
\ge4c_J(1-t)^2.
\end{align}
This proves Equation~\eqref{eq:init_orders}. The conditions concern only the pre-optimizer gradient at initialization.

\paragraph{A checkable output-head condition.}
An additive output bias provides a sufficient condition without requiring every network direction to be nondegenerate. Write its contribution as $Sb$, where $b$ is a trainable bias vector and $S$ is its fixed replication map across output coordinates. The bias block alone gives
\begin{equation}
G_{\theta_0}(t)\ge\Ex\norm{\nabla_b\ell_x(t)}^2
=4\Ex\norm{S^\top r_0(t)}^2.
\label{eq:bias_certificate}
\end{equation}
For a zero-initialized direct output, this is the time-independent lower bound $4\Ex\norm{S^\top X}^2$, positive whenever the projected data are nonzero with positive probability. In the identity-residual case, independent Gaussian noise makes $\Ex\norm{S^\top(E-X)}^2>0$ for every nonzero $S$. The JiT implementation uses a zero-initialized patch output layer with a shared bias, so this check concerns sums of target patches. Appendix~\ref{app:init_head_conditions} derives both cases and states the corresponding endpoint condition for the nonzero-initialized U-Net head.

\section{Terminal Amplification}
\label{sec:theory}

\subsection{Sampler-dependent divergence at initialization}
\label{sec:init_divergence}

Substituting Equation~\eqref{eq:init_orders} into the exact moment identity~\eqref{eq:optimizer_second_moments} gives the endpoint criterion below.

\begin{theorem}[Sampler-dependent initialization-time divergence]
\label{thm:init_divergence}
Suppose Assumption~\ref{ass:init_regimes} holds and training timesteps have density $q$. Then, for every $b\in(t_0,1)$,
\begin{equation}
M_v(q)
=\int_0^1q(t)(1-t)^{-4}G_{\theta_0}(t)dt
\ge4c_J\int_{t_0}^{b}q(t)(1-t)^{-p}dt,
\label{eq:init_gradient_orders}
\end{equation}
where $p=4$ in Regime A and $p=2$ in Regime B. More explicitly, suppose that for some $t_1\in[t_0,1)$, $c_q>0$, and $\beta\ge0$,
\begin{equation}
q(t)\ge c_q(1-t)^\beta
\quad\text{for a.e. }t\in[t_1,1).
\label{eq:sampler_endpoint_lower_bound}
\end{equation}
If $\beta\le p-1$, then $M_v(q)=\infty$. Hence Uniform sampling, for which $q_{\mathrm U}(t)=1$ and $\beta=0$, diverges in both regimes.
\end{theorem}

The proof is in Appendix~\ref{app:proof_divergence}. The thresholds are $\beta\le3$ in Regime A and $\beta\le1$ in Regime B, obtained from lower bounds on the conditional mismatch moment of orders $(1-t)^{-4}$ and $(1-t)^{-2}$, respectively. These are sufficient divergence conditions; larger $\beta$ alone does not establish integrability. At a later parameter state, Equation~\eqref{eq:optimizer_second_moments} remains exact, but integrability depends on the learned $G_\theta(t)$.

\subsection{Alignment and endpoint-suppressing sampling}
\label{sec:alignment_sampling}

\paragraph{Prediction--loss alignment.}
For an $x$-prediction head, alignment removes the conversion weight from training: $\ell_x(t)=(1-t)^2\ell_v(t)$ and
\begin{equation}
M_x(q)=\int_0^1q(t)G_\theta(t)dt.
\label{eq:aligned_identity_moment}
\end{equation}

\begin{proposition}[Aligned objective]
\label{prop:aligned_regularity}
If $C_G(\theta):=\operatorname*{ess\,sup}_{t\in(0,1)}G_\theta(t)<\infty$, then for every timestep density $q$,
\begin{equation}
M_x(q)\le C_G(\theta)<\infty.
\label{eq:aligned_bound}
\end{equation}
\end{proposition}

The bound follows immediately from $\int_0^1q(t)dt=1$. Appendix~\ref{app:alignment_proof} gives sufficient Jacobian--residual moment conditions and analogous bounds for aligned velocity MSE and BCE-with-logits. The implied velocity in Equation~\eqref{eq:implied_velocity} is still used at inference.

\paragraph{Logit-Normal endpoint suppression.}
Let $U\sim\mathcal N(m,s^2)$ with $m\in\mathbb R$ and $s>0$, and set $T=\sigma(U)=(1+e^{-U})^{-1}$. The induced density on $(0,1)$ is
\begin{equation}
q_{\mathrm{LN}}(t)
=\frac{1}{s\sqrt{2\pi}\,t(1-t)}
\exp\left[-\frac{(\operatorname{logit}(t)-m)^2}{2s^2}\right].
\label{eq:lognormal_density}
\end{equation}

\begin{proposition}[Logit-Normal integrability under polynomial endpoint growth]
\label{prop:lognormal_integrability}
Fix $\theta$. Suppose that for some $t_0\in(0,1)$, $C<\infty$, and $p\ge0$,
\begin{equation}
G_\theta(t)\le C(1-t)^{-p},\qquad t\in[t_0,1),
\label{eq:lognormal_growth}
\end{equation}
and $\int_0^{t_0}q_{\mathrm{LN}}(t)(1-t)^{-4}G_\theta(t)dt<\infty$. Then $M_v(q_{\mathrm{LN}})<\infty$.
\end{proposition}

\paragraph{Proof sketch.}
The logistic change of variables $t=\sigma(u)$ gives
\begin{equation}
\begin{aligned}
&\int_{t_0}^1q_{\mathrm{LN}}(t)(1-t)^{-4}G_\theta(t)dt \\
&\quad\le C\,\mathbb E\left[(1-\sigma(U))^{-(p+4)}\mathbf 1\{U\ge\operatorname{logit}(t_0)\}\right].
\end{aligned}
\label{eq:lognormal_tail_bound}
\end{equation}
Because $(1-\sigma(U))^{-1}=1+e^U$ and the Gaussian moment-generating function satisfies $\mathbb E[e^{rU}]=e^{rm+r^2s^2/2}<\infty$ for every finite $r$, the right-hand side is finite. Adding the assumed nonterminal integral proves the proposition. In particular, $C_G(\theta)<\infty$ implies
\begin{equation}
M_v(q_{\mathrm{LN}})
\le C_G(\theta)\,\mathbb E[(1-\sigma(U))^{-4}]<\infty.
\label{eq:lognormal_bound}
\end{equation}
The full proof is in Appendix~\ref{app:lognormal}. Alignment removes the mismatch factor pointwise; Logit-Normal sampling retains it but can make its integrated contribution finite by suppressing the endpoint.

\section{Experiments}
\label{sec:experiments}

Guided by our theoretical analysis, we organize the experiments around two questions. \textbf{(i) Endpoint amplification and training quality:} Does prediction--loss mismatch produce the endpoint gradient amplification predicted by the theory, and how does this amplification manifest in model fitting and generation quality? Section~\ref{sec:empirical_amplification} examines these questions through direct gradient measurements and controlled training interventions. \textbf{(ii) Alignment for sampler-robust training:} Does prediction--loss alignment support effective training under both Uniform and Logit-Normal sampling, and can Logit-Normal sampling mitigate the training failure observed under mismatch? Section~\ref{sec:empirical_robustness} evaluates both effects through objective--sampler comparisons. Full experimental protocols and supplementary MIMO results are provided in Appendices~\ref{app:experimental_details} and~\ref{app:mimo}.

\subsection{Endpoint exposure and gradient amplification}
\label{sec:empirical_amplification}

We first tested the predicted growth of gradient second moments as the timestep approached the endpoint at initialization. We then evaluated how restricting the corresponding endpoint weighting affected generation quality.

\paragraph{Direct gradient measurements at initialization.}
To test the two initialization regimes, we used a two-hidden-layer gated MLP with a zero-initialized output layer as a direct signal predictor (Regime A). Adding an identity skip from the input gave a residual predictor (Regime B), so the two networks initially produced $0$ and $Z_t$, respectively. Both used paired branch parameters and identical data--noise pairs; implementation details are provided in Appendix~\ref{app:toy_initialization}.

For each frozen network, we independently backpropagated the aligned and mismatched losses to measure individual-example gradient second moments. Table~\ref{tab:toy_initialization} shows that mismatched moments grew sharply as $\epsilon=1-t$ decreased in both regimes, whereas aligned moments remained approximately constant in A and decreased in B. In Regime B, the approximately quadratic decay of the aligned moment was insufficient to offset the mismatch multiplier. These measurements supported the predicted endpoint amplification in both initialization regimes and its removal by alignment over the tested endpoint distances.

\begin{table}[hbt!]
\small
\centering
\caption{Directly measured gradient second moments at initialization, with mismatch as the tested objective and alignment as the control.}
\label{tab:toy_initialization}
\setlength{\tabcolsep}{5pt}
\renewcommand{\arraystretch}{1.15}
\begin{tabular}{lccc}
\toprule
Initialization / loss & $\epsilon=10^{-2}$ & $\epsilon=10^{-4}$ & $\epsilon=10^{-6}$ \\
\midrule
A: mismatch & $(2.92\pm0.03)\times10^{7}$ & $(2.93\pm0.03)\times10^{15}$ & $(2.93\pm0.03)\times10^{23}$ \\
A: aligned & $0.292\pm0.003$ & $0.293\pm0.003$ & $0.293\pm0.003$ \\
\midrule
B: mismatch & $(5.93\pm0.05)\times10^{3}$ & $(5.95\pm0.05)\times10^{7}$ & $(5.95\pm0.05)\times10^{11}$ \\
B: aligned & $(5.93\pm0.05)\times10^{-5}$ & $(5.95\pm0.05)\times10^{-9}$ & $(5.95\pm0.05)\times10^{-13}$ \\
\bottomrule
\end{tabular}

\end{table}

\Needspace{12\baselineskip}
\begin{wraptable}[11]{r}{0.43\textwidth}
\vspace{-0.5\baselineskip}
\centering
\caption{Binary-MNIST endpoint-cap sweep. Weight moments are analytical; FID-50K is mean $\pm$ s.d. over three seeds.}
\label{tab:bmnist_cap}
\small
\setlength{\tabcolsep}{4pt}
\begin{tabular}{ccc}
\toprule
Cap $C$ & $\Ex\widetilde w_C^2$ & FID-50K \\
\midrule
$10^2$ & $3.69$ & $16.02\pm2.13$ \\
$10^4$ & $33.67$ & $27.18\pm4.25$ \\
$10^6$ & $333.67$ & $614.37\pm10.10$ \\
$10^8$ & $3333.67$ & $1625.48\pm70.82$ \\
\bottomrule
\end{tabular}
\end{wraptable}

\paragraph{Endpoint weighting and generation quality.}
Following the initialization-time endpoint measurements, we examined whether restricting the endpoint weighting responsible for mismatch amplification improved generation quality. We trained a conditional U-Net on Binary MNIST under Uniform sampling, capping the inverse-square mismatch weight at $C$ and normalizing it to unit mean. Smaller caps imposed stronger restrictions on endpoint weighting, while the normalization held its mean scale fixed (Appendix~\ref{app:cap_sweep}).

Table~\ref{tab:bmnist_cap} shows that increasing $C$ from $10^2$ to $10^8$ raised the weight second moment approximately $10^3$-fold and worsened FID from $16.02$ to $1625.48$. Stronger restrictions therefore yielded substantially better generation, even though all tested weights were bounded. This supported restricting endpoint weighting as an effective intervention in this setting: the strength of the retained weighting mattered for generation quality beyond its numerical finiteness. The cap jointly changed the strength of endpoint weighting and the time profile of supervision.
\WFclear

\subsection{Prediction--loss alignment for sampler-robust training}
\label{sec:empirical_robustness}

To test the predicted distinction between suppressing endpoint exposure and removing conversion-induced amplification, we compared timestep sampling and prediction--loss alignment in three settings. A Gaussian MLP examined the sampler-dependent gradient behavior of the mismatched objective. Conditional U-Net and JiT-B/4 experiments then tested whether this sensitivity appeared in image generation, and whether alignment supported effective training under both Uniform and Logit-Normal sampling. Together, these comparisons evaluated the practical claim of sampler-robust training across binary and continuous data.

\Needspace{13\baselineskip}
\begin{wrapfigure}[13]{r}{0.60\textwidth}
\vspace{-0.6\baselineskip}
\centering
\includegraphics[width=\linewidth]{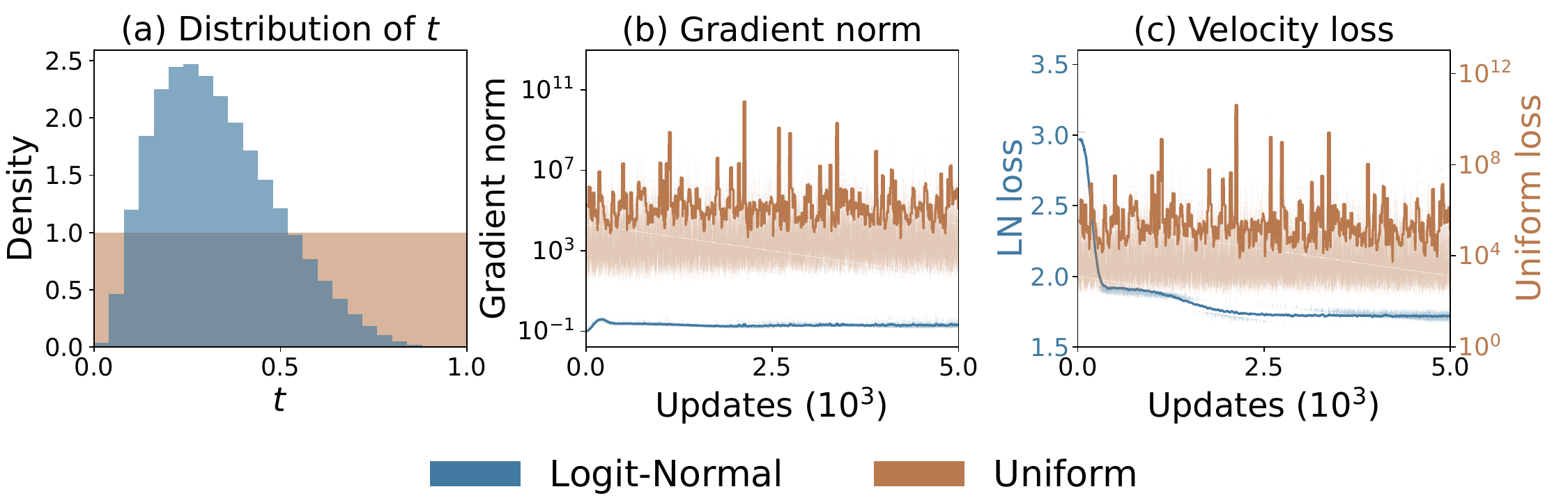}
\caption{Gaussian sampler control. (a) Sampler densities. (b,c) Raw trajectories (faint) and 25-update means (solid) across three paired seeds; loss axes are sampler-specific.}
\label{fig:toy_gaussian}
\end{wrapfigure}

\begingroup
\raggedright
\paragraph{Sampler control.}
To examine whether sampling moderated the gradients produced by the singular conversion, we held the Gaussian MLP and raw mismatched objective fixed and paired initialization and data--noise streams across samplers. Median gradient norms were of order $10^3$ under Uniform and $10^{-1}$ under Logit-Normal, with much larger Uniform excursions (Figure~\ref{fig:toy_gaussian}). This contrast supported the predicted role of endpoint exposure: changing how often terminal timesteps were sampled substantially moderated the observed gradients while retaining the mismatched conversion. Appendix~\ref{app:toy_gaussian} gives the protocol.
\par\WFclear
\endgroup

\begin{figure}[!t]
\centering
\includegraphics[width=0.94\linewidth,trim=4.7bp 8.5bp 6.1bp 8.5bp,clip]{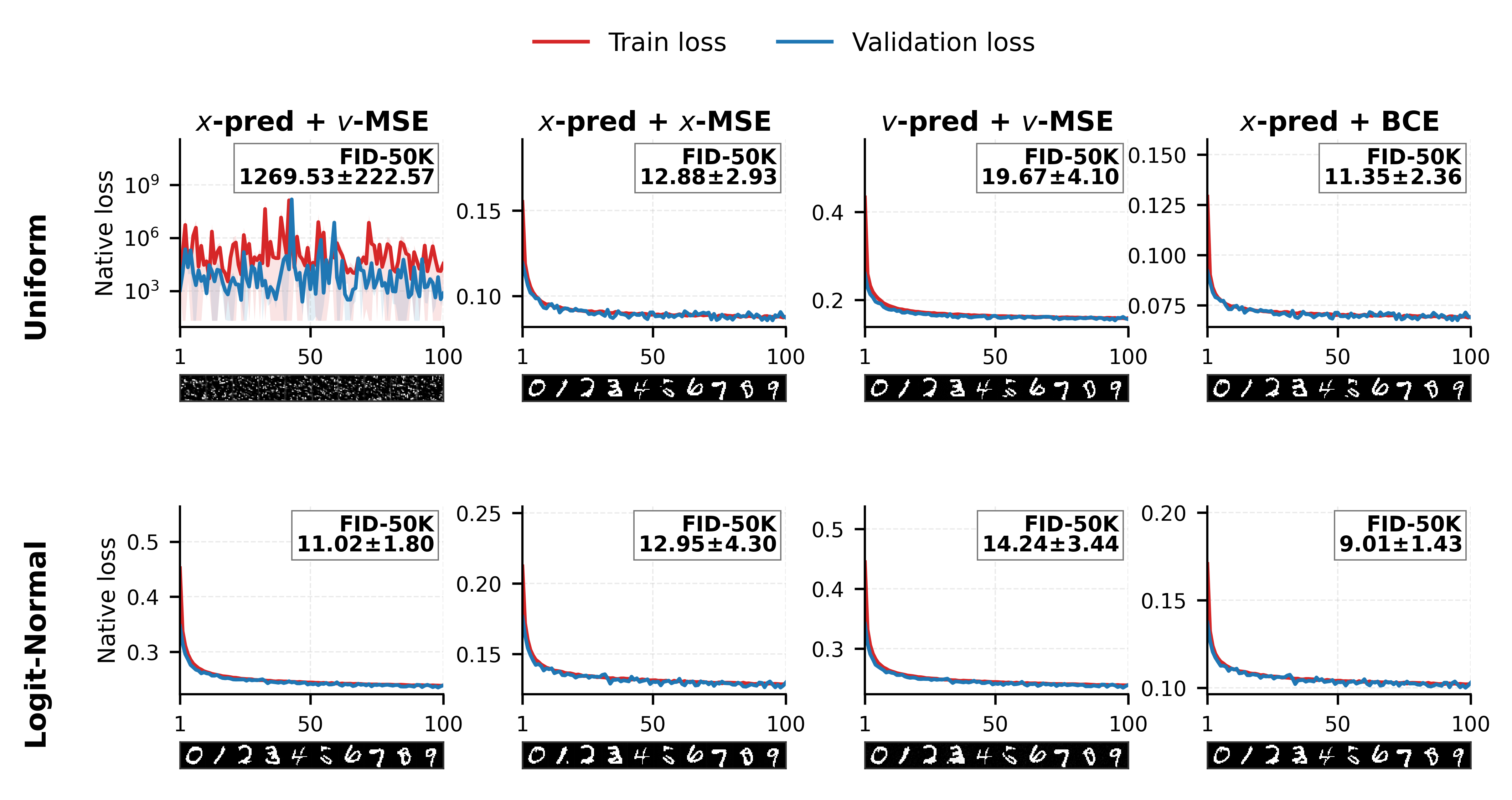}
\caption{Binary-MNIST objective--sampler matrix: FID-50K and curves summarize three training seeds (mean $\pm$ s.d.), with samples below. Rows vary samplers; columns vary prediction--loss pairings. Native losses are comparable only within each objective.}
\label{fig:bmnist_corrected}
\end{figure}

\paragraph{Effect of timestep sampling.}
To test whether sampler sensitivity also appeared in generation quality, we evaluated a conditional U-Net \citep{ronneberger2015unet} on Binary MNIST \citep{lecunmnist} and JiT-B/4 \citep{li2025back} on Tiny-ImageNet. Each within-backbone comparison changed the sampler while retaining the prediction--loss pairing and evaluation protocol. Under mismatched $x$-prediction+$v$-MSE, replacing Uniform with Logit-Normal sampling reduced FID-50K \citep{heusel2017gans} from $1269.53\pm222.57$ to $11.02\pm1.80$ for U-Net and from $390.50$ to $14.56$ for JiT.

Figures~\ref{fig:bmnist_corrected} and~\ref{fig:jit_dynamics} show generation collapse under Uniform and recognizable samples under Logit-Normal, together with the corresponding loss trajectories. The improvement in both backbones supported boundary-suppressing sampling as an effective remedy for the observed mismatch-related training difficulties. It also illustrated how a recipe with strong endpoint amplification could nevertheless achieve good generation quality under a suitable sampler. The JiT comparison retained the same small denominator floor in both sampler cells. Full protocols are given in Appendices~\ref{app:bmnist_protocol} and~\ref{app:jit_details}.

\Needspace{5\baselineskip}
\paragraph{Effect of prediction--loss alignment.}
To test alignment as a remedy that did not rely on boundary-suppressing sampling, we replaced velocity MSE with signal MSE while retaining the direct $x$-prediction backbone under each sampler. Aligned U-Net obtained FID $12.88\pm2.93$ under Uniform and $12.95\pm4.30$ under Logit-Normal sampling. Aligned JiT obtained $15.85$ and $24.32$, respectively. Both aligned configurations therefore produced useful samples under Uniform as well as Logit-Normal sampling, in contrast to the strong sampler dependence of mismatch. The complete objective--sampler matrices also included direct $v$-MSE and, for U-Net, aligned BCE; both alternatives remained trainable under the two samplers. These within-backbone comparisons supported alignment's reduced sensitivity to sampler replacement across the evaluated continuous and binary generation tasks.

\paragraph{Generation quality and sampler robustness.}
JiT achieved its best observed FID with mismatch plus Logit-Normal sampling. Alignment's advantage was effective generation under both tested samplers, whereas mismatch depended strongly on sampler choice. Taken together, these results supported the practical benefit of removing the conversion-induced amplification: effective training was maintained across the two tested timestep distributions. The lower FID of mismatch plus Logit-Normal further showed that best-case generation quality and robustness to sampler replacement were distinct aspects of the training recipe.

\begin{figure}[!t]
\centering
\includegraphics[width=\linewidth,trim=5.8bp 5.2bp 6.5bp 8.5bp,clip]{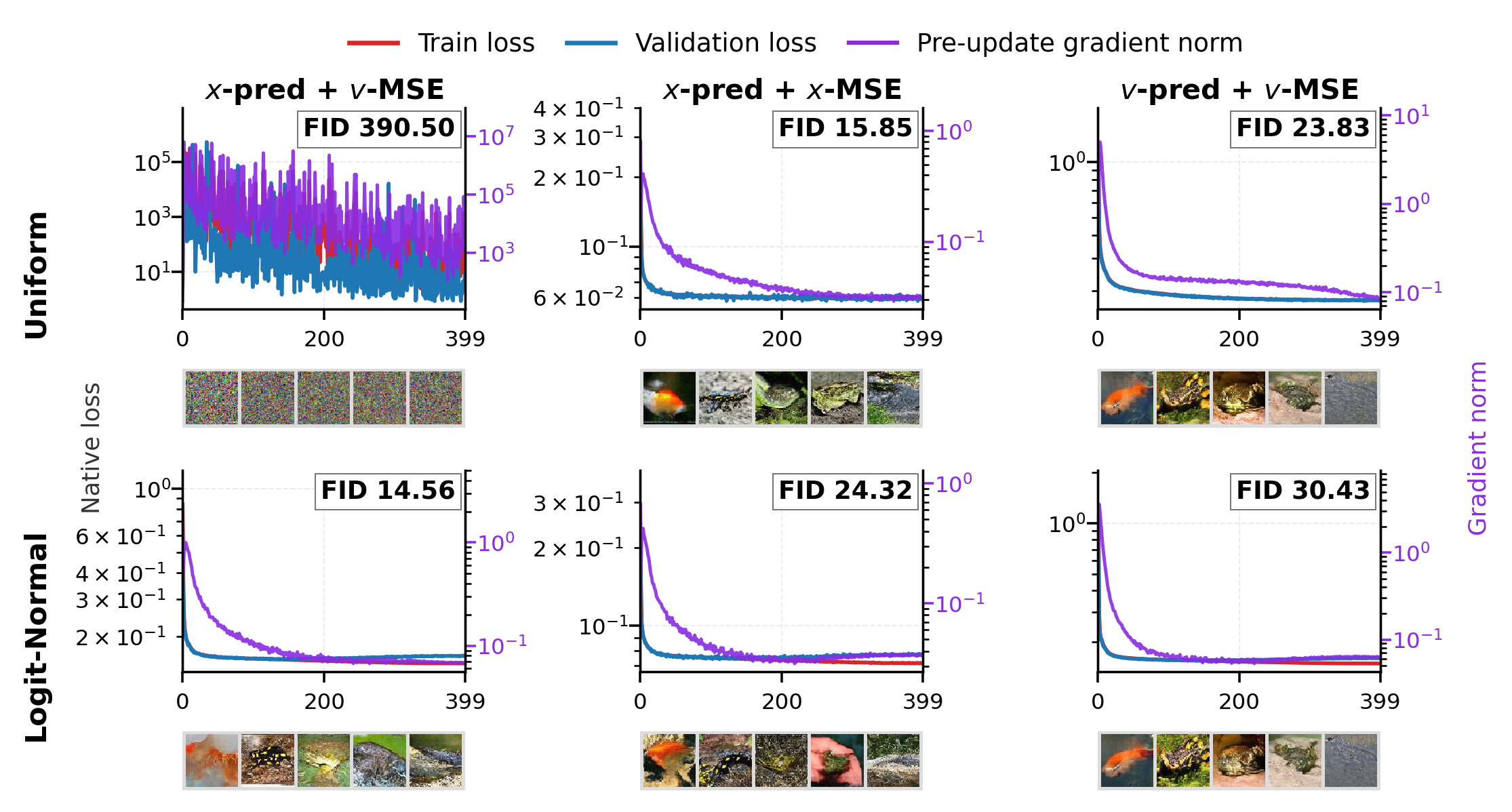}
\caption{JiT-B/4 objective--sampler matrix. Each cell reports native train/validation losses, mean pre-update full-parameter gradient norm, FID-50K, and the first row of its epoch-400 sample grid. Each cell is one run. Native loss values are comparable within, not across, objectives.}
\label{fig:jit_dynamics}
\end{figure}

\paragraph{Supplementary training diagnostics.}
The paired JiT early-training comparison retained the initialization and data--noise--timestep stream while changing the objective. Raw mismatch produced a median pre-update gradient norm of order $10^2$, compared with $10^{-1}$ for alignment, and retained higher signal MSE (Appendix~\ref{app:jit_early}). Frozen trained-checkpoint measurements additionally showed endpoint gradient amplification on the evaluated timestep grid (Appendix~\ref{app:frozen_endpoint}). These measurements complemented the initialization analysis by showing that endpoint amplification was also observable during early fitting and at the evaluated trained states, connecting the local gradient diagnostics with the broader training comparisons.

\FloatBarrier
\Needspace{6\baselineskip}
\section{Discussion and Limitations}
\label{sec:limitations}

\paragraph{Why a singular recipe can still work.}
The conversion amplifies gradients at a given timestep; the sampler determines how often training encounters it. Under the growth conditions in Section~\ref{sec:alignment_sampling}, endpoint suppression makes the aggregate second moment finite despite the singular pointwise multiplier. The controlled comparisons illustrated how these effects can coexist: changing the sampler restored useful generation without removing the conversion.

\paragraph{Generation quality and sampler robustness.}
The JiT results suggest evaluating both the best observed FID and sensitivity to sampler replacement. Alignment removes the conversion multiplier, yielding the uniform-in-sampler moment bound in Proposition~\ref{prop:aligned_regularity}. It retains the prediction architecture and ODE velocity formula, but can change the learned predictor and final sample quality. Its practical value therefore need not coincide with the lowest FID under a particular sampler.

\paragraph{Beyond finite second moments.}
The endpoint-cap sweep shows why integrability is a starting point rather than a complete account of training quality. All tested caps bounded the conversion weight; at a fixed state with bounded $G_\theta(t)$, this also bounds the gradient second moment. Nevertheless, larger caps produced substantially worse generation. Thus, the strength of the retained endpoint weighting matters in addition to whether it is bounded. Because capping changes both the objective's time profile and the stochastic-gradient sequence, these results motivate examining their joint effects on fitting and generation quality.

\paragraph{Implications for sampler design and reporting.}
Timestep sampling should be interpreted together with the prediction--loss parameterization. For fixed $\theta$ and $T\sim q$ sampled independently of the data and noise, Equation~\eqref{eq:formal_conversion} gives
\begin{equation}
\Ex[\ell_v(T)]=\int_0^1 \frac{q(t)}{(1-t)^2}\,\Ex[\ell_x(t)]\,dt,
\qquad
\Ex[\ell_x(T)]=\int_0^1 q(t)\,\Ex[\ell_x(t)]\,dt.
\label{eq:discussion_time_weight}
\end{equation}
The outer expectations include timestep sampling; the fixed-$t$ expectations follow Section~3.1. Mismatch therefore carries the signal-space time weight $q(t)a(t)$, whereas alignment retains $q(t)$ alone. Sampling remains $T\sim q$; only the objective weighting changes. We view aligned signal space as a useful common reference for future sampler studies: removing the conversion weight makes the sampler's allocation of signal-space supervision easier to interpret.

In practice, evaluations should report the sampler, endpoint truncation or denominator handling, and pre-clipping gradient diagnostics alongside generation metrics. This distinguishes a favourable sampler--objective combination from robustness to endpoint exposure. With conversion-induced amplification removed, learning-rate selection and optimizer tuning remain part of the training recipe.

\paragraph{Limitations.}
The divergence theorem is a sufficient result at initialization under the stated output-gradient conditions, rather than a convergence theorem for an optimizer. The frozen diagnostic covers two trained checkpoints and a finite set of inputs and timesteps; it cannot certify trained-state asymptotics. The full JiT matrix has one run per cell, whereas Binary MNIST used three training seeds and a dataset-specific FID feature extractor. Finally, the formal results concern the stated linear Gaussian interpolation. Other paths require analyzing their own prediction-to-velocity conversion.

\section{Conclusion}

We investigated how $x$-prediction+$v$-loss can work well despite its singular endpoint amplification. Under the stated initialization conditions, its stochastic-gradient second moment diverges under Uniform sampling; under the stated growth conditions, boundary-suppressing sampling restores integrability without removing pointwise amplification. Alignment removes the conversion factor and, under a uniform moment bound, yields a finite second moment for every timestep density. Continuous and binary experiments corroborated mismatch's sampler sensitivity and aligned objectives' trainability across the tested samplers. These findings reconcile amplification with empirical success and support alignment as a principled route to sampler-robust training.

\Needspace{8\baselineskip}
\subsection*{AI use statement}

Generative AI tools were used during revision for language editing, code assistance, mathematical adversarial checking, and organization of experimental diagnostics. The authors independently re-derived the mathematical statements, reviewed all generated text, tested code changes, inspected experimental outputs, and checked cited sources. Generative AI tools were not treated as scientific sources or final arbiters of correctness. The authors take responsibility for all claims, proofs, code, and artifacts in this work.

\subsection*{Reproducibility statement}

Appendix~\ref{app:experimental_details} specifies datasets, architectures, objectives, samplers, seeds, and evaluation procedures. The supplementary material will include the Binary-MNIST harness, JiT diagnostics, and machine-readable per-run summaries.

\bibliography{references}

@inproceedings{loshchilov2017adamw,
  title={Decoupled Weight Decay Regularization},
  author={Loshchilov, Ilya and Hutter, Frank},
  booktitle={International Conference on Learning Representations},
  year={2019},
  url={https://openreview.net/forum?id=Bkg6RiCqY7}
}

@inproceedings{lou2024discrete,
  title={Discrete Diffusion Modeling by Estimating the Ratios of the Data Distribution},
  author={Lou, Aaron and Meng, Chenlin and Ermon, Stefano},
  booktitle={Proceedings of the 41st International Conference on Machine Learning},
  pages={32819--32848},
  volume={235},
  series={Proceedings of Machine Learning Research},
  publisher={PMLR},
  year={2024},
  url={https://proceedings.mlr.press/v235/lou24a.html}
}

@article{li2025back,
  title={Back to Basics: Let Denoising Generative Models Denoise},
  author={Li, Tianhong and He, Kaiming},
  journal={arXiv preprint arXiv:2511.13720},
  year={2025}
}

@article{hu2026elf,
  title={{ELF}: Embedded Language Flows},
  author={Hu, Keya and Qiu, Linlu and Lu, Yiyang and Zhao, Hanhong and Li, Tianhong and Kim, Yoon and Andreas, Jacob and He, Kaiming},
  journal={arXiv preprint arXiv:2605.10938},
  year={2026}
}

@inproceedings{chenanalog,
  title={Analog Bits: Generating Discrete Data Using Diffusion Models with Self-Conditioning},
  author={Chen, Ting and Zhang, Ruixiang and Hinton, Geoffrey},
  booktitle={The Eleventh International Conference on Learning Representations},
  year={2023},
  url={https://openreview.net/forum?id=3itjR9QxFw}
}

@inproceedings{ho2020denoising,
  title={Denoising Diffusion Probabilistic Models},
  author={Ho, Jonathan and Jain, Ajay and Abbeel, Pieter},
  booktitle={Advances in Neural Information Processing Systems},
  volume={33},
  pages={6840--6851},
  year={2020}
}

@inproceedings{lipmanflow,
  title={Flow Matching for Generative Modeling},
  author={Lipman, Yaron and Chen, Ricky TQ and Ben-Hamu, Heli and Nickel, Maximilian and Le, Matt},
  booktitle={The Eleventh International Conference on Learning Representations},
  year={2023}
}

@inproceedings{peebles2023scalable,
  title={Scalable Diffusion Models with Transformers},
  author={Peebles, William and Xie, Saining},
  booktitle={Proceedings of the IEEE/CVF International Conference on Computer Vision},
  pages={4195--4205},
  year={2023}
}

@inproceedings{hong2025soft,
  title={Soft Graph Transformer for {MIMO} Detection},
  author={Hong, Jiadong and Liu, Lei and Bian, Xinyu and Wang, Wenjie and Zhang, Zhaoyang},
  booktitle={ICASSP 2026-2026 IEEE International Conference on Acoustics, Speech and Signal Processing (ICASSP)},
  pages={21421--21425},
  year={2026},
  organization={IEEE},
  doi={10.1109/ICASSP55912.2026.11463763}
}

@inproceedings{liurflow,
  title={Flow Straight and Fast: Learning to Generate and Transfer Data with Rectified Flow},
  author={Liu, Xingchao and Gong, Chengyue and Liu, Qiang},
  booktitle={The Eleventh International Conference on Learning Representations},
  year={2023}
}

@inproceedings{esser2024scaling,
  title={Scaling Rectified Flow Transformers for High-Resolution Image Synthesis},
  author={Esser, Patrick and Kulal, Sumith and Blattmann, Andreas and Entezari, Rahim and M{\"u}ller, Jonas and Saini, Harry and Levi, Yam and Lorenz, Dominik and Sauer, Axel and Boesel, Frederic and Podell, Dustin and Dockhorn, Tim and English, Zion and Rombach, Robin},
  booktitle={Proceedings of the 41st International Conference on Machine Learning},
  pages={12606--12633},
  volume={235},
  series={Proceedings of Machine Learning Research},
  publisher={PMLR},
  year={2024},
  url={https://proceedings.mlr.press/v235/esser24a.html}
}

@article{dieleman2022cdcd,
  title={Continuous Diffusion for Categorical Data},
  author={Dieleman, Sander and Sartran, Laurent and Roshannai, Arman and Savinov, Nikolay and Ganin, Yaroslav and Richemond, Pierre H and Doucet, Arnaud and Strudel, Robin and Dyer, Chris and Durkan, Conor and Hawthorne, Curtis and Leblond, R{\'e}mi and Grathwohl, Will and Adler, Jonas},
  journal={arXiv preprint arXiv:2211.15089},
  year={2022}
}

@article{austin2021d3pm,
  title={Structured Denoising Diffusion Models in Discrete State-Spaces},
  author={Austin, Jacob and Johnson, Daniel D and Ho, Jonathan and Tarlow, Daniel and van den Berg, Rianne},
  journal={Advances in Neural Information Processing Systems},
  volume={34},
  pages={17981--17993},
  year={2021}
}

@inproceedings{campbell2024generative,
  title={Generative Flows on Discrete State-Spaces: Enabling Multimodal Flows with Applications to Protein Co-Design},
  author={Campbell, Andrew and Yim, Jason and Barzilay, Regina and Rainforth, Tom and Jaakkola, Tommi},
  booktitle={Proceedings of the 41st International Conference on Machine Learning},
  pages={5453--5512},
  volume={235},
  series={Proceedings of Machine Learning Research},
  publisher={PMLR},
  year={2024},
  url={https://proceedings.mlr.press/v235/campbell24a.html}
}

@ARTICLE{lecunmnist,
  author={LeCun, Yann and Bottou, L{\'e}on and Bengio, Yoshua and Haffner, Patrick},
  journal={Proceedings of the IEEE}, 
  title={Gradient-Based Learning Applied to Document Recognition}, 
  year={1998},
  volume={86},
  number={11},
  pages={2278-2324},
  doi={10.1109/5.726791}}

@inproceedings{heusel2017gans,
  title={{GANs} Trained by a Two Time-Scale Update Rule Converge to a Local Nash Equilibrium},
  author={Heusel, Martin and Ramsauer, Hubert and Unterthiner, Thomas and Nessler, Bernhard and Hochreiter, Sepp},
  booktitle={Advances in Neural Information Processing Systems},
  volume={30},
  year={2017}
}

@inproceedings{ronneberger2015unet,
  title={{U-Net}: Convolutional Networks for Biomedical Image Segmentation},
  author={Ronneberger, Olaf and Fischer, Philipp and Brox, Thomas},
  booktitle={International Conference on Medical Image Computing and Computer-Assisted Intervention},
  pages={234--241},
  year={2015},
  organization={Springer}
}

@article{eijkelboom2024vfm,
  title={Variational Flow Matching for Graph Generation},
  author={Eijkelboom, Floor and Bartosh, Grigory and Naesseth, Christian A and Welling, Max and van de Meent, Jan-Willem},
  journal={Advances in Neural Information Processing Systems},
  volume={37},
  pages={11735--11764},
  year={2024},
  doi={10.52202/079017-0374}
}

@inproceedings{karras2022edm,
  title={Elucidating the Design Space of Diffusion-Based Generative Models},
  author={Karras, Tero and Aittala, Miika and Aila, Timo and Laine, Samuli},
  booktitle={Advances in Neural Information Processing Systems},
  volume={35},
  pages={26565--26577},
  year={2022},
  doi={10.52202/068431-1926}
}

@inproceedings{kingma2023understanding,
  title={Understanding Diffusion Objectives as the {ELBO} with Simple Data Augmentation},
  author={Kingma, Diederik P. and Gao, Ruiqi},
  booktitle={Advances in Neural Information Processing Systems},
  volume={36},
  pages={65484--65516},
  year={2023},
  doi={10.52202/075280-2858}
}

@inproceedings{hang2023minsnr,
  title={Efficient Diffusion Training via Min-{SNR} Weighting Strategy},
  author={Hang, Tiankai and Gu, Shuyang and Li, Chen and Bao, Jianmin and Chen, Dong and Hu, Han and Geng, Xin and Guo, Baining},
  booktitle={Proceedings of the IEEE/CVF International Conference on Computer Vision},
  pages={7441--7451},
  month={October},
  year={2023}
}

@article{jiang2026sgdit,
  title={Soft Graph Diffusion Transformer for {MIMO} Detection},
  author={Jiang, Nan and Hong, Jiadong and Liu, Lei and Bian, Xinyu and Wang, Wenjie and Zhang, Zhaoyang},
  journal={arXiv preprint arXiv:2605.00449},
  year={2026},
  doi={10.48550/arXiv.2605.00449},
  url={https://arxiv.org/abs/2605.00449}
}

@inproceedings{kingma2015adam,
  title={Adam: A Method for Stochastic Optimization},
  author={Kingma, Diederik P. and Ba, Jimmy},
  booktitle={International Conference on Learning Representations},
  year={2015},
  url={https://arxiv.org/abs/1412.6980}
}

@techreport{le2015tinyimagenet,
  title={Tiny ImageNet Visual Recognition Challenge},
  author={Le, Ya and Yang, Xuan},
  year={2015},
  institution={Stanford University},
  type={{CS231N} Course Project Report},
  url={https://cs231n.stanford.edu/reports/2015/pdfs/yle_project.pdf}
}

@article{ho2022classifierfree,
  title={Classifier-Free Diffusion Guidance},
  author={Ho, Jonathan and Salimans, Tim},
  journal={arXiv preprint arXiv:2207.12598},
  year={2022},
  url={https://arxiv.org/abs/2207.12598}
}

@misc{obukhov2020torchfidelity,
  title={High-Fidelity Performance Metrics for Generative Models in {PyTorch}},
  author={Obukhov, Anton and Seitzer, Maximilian and Wu, Po-Wei and Zhydenko, Semen and Kyl, Jonathan and Lin, Elvis Yu-Jing},
  year={2020},
  publisher={Zenodo},
  doi={10.5281/zenodo.3786539},
  url={https://github.com/toshas/torch-fidelity}
}
\bibliographystyle{plainnat}

\appendix
\section{Proofs and Additional Analysis}
\label{app:proofs}

\subsection{Fixed-time population target equivalence}
\label{app:flow_target_proof}

\begin{proposition}[Signal regression recovers the fixed-time population flow target]
\label{prop:flow_target}
Fix $t\in(0,1)$ and let $m_t(z,c)$ be any version of $\Ex[X\mid Z_t=z,C=c]$, with $c$ omitted in the unconditional setting. The unrestricted population minimizer of signal-space MSE is $\hat X^*(z,c,t)=m_t(z,c)$, uniquely up to null sets of $(Z_t,C)$ (or $Z_t$ unconditionally). Moreover, a version of the marginal flow-matching velocity is
\begin{equation}
v^*(z,c,t):=\Ex[X-E\mid Z_t=z,C=c]
=\frac{m_t(z,c)-z}{1-t},
\label{eq:population_flow_target}
\end{equation}
for $P_{(Z_t,C)}$-almost every $(z,c)$. Since $X-E\in L^2$, this conditional expectation is also the unrestricted velocity-MSE minimizer.
\end{proposition}

\begin{proof}
For fixed $t$, the conditional-expectation orthogonality identity gives, for every measurable predictor $f$ with finite risk,
\begin{equation}
\Ex\norm{f(Z_t,C,t)-X}^2
=\Ex\norm{f(Z_t,C,t)-m_t(Z_t,C)}^2
+\Ex\norm{m_t(Z_t,C)-X}^2,
\end{equation}
where $m_t(z,c)$ is any version of $\Ex[X\mid Z_t=z,C=c]$. Hence signal-space MSE is minimized by $f=m_t$, uniquely up to $P_{(Z_t,C)}$-null sets. From $Z_t=tX+(1-t)E$ we also have, conditional on $(Z_t,C)=(z,c)$,
\begin{equation}
X-E=X-\frac{z-tX}{1-t}=\frac{X-z}{1-t}.
\end{equation}
Taking conditional expectations yields
\begin{equation}
\Ex[X-E\mid Z_t=z,C=c]=\frac{m_t(z,c)-z}{1-t},
\end{equation}
for $P_{(Z_t,C)}$-almost every $(z,c)$, which is exactly the implied velocity of the population-optimal signal predictor. Since $X-E\in L^2$, applying the same orthogonality identity with target $X-E$ shows that this conditional expectation is the unrestricted velocity-MSE minimizer. In the unconditional setting, all occurrences of $C$ and $c$ are omitted.
\end{proof}

\subsection{Sufficient conditions from the output bias}
\label{app:init_head_conditions}

The gradient lower bounds in Assumption~\ref{ass:init_regimes} can be checked through a subset of trainable parameters. Suppose the prediction contains an additive bias $Sb$, where $S\in\mathbb R^{D\times d_b}$ is a fixed nonzero matrix and the remaining prediction does not depend on $b$. For the summed signal loss,
\begin{equation}
\nabla_b\ell_x(t)=2S^\top r_0(t),\qquad
G_{\theta_0}(t)\ge4\Ex\norm{S^\top r_0(t)}^2.
\end{equation}
Sharing a bias across pixels or patches changes $S$; it does not invalidate this parameter-block lower bound.

\paragraph{Zero-initialized direct output.}
If $\hat X_{\theta_0}(z,t)=0$ for all inputs, then $r_0(t)=-X$ and
\begin{equation}
G_{\theta_0}(t)\ge4\Ex\norm{S^\top X}^2.
\label{eq:zero_head_bias}
\end{equation}
Thus Regime A holds whenever $\Pr(S^\top X\ne0)>0$. For independent per-coordinate biases, $S=I_D$ and the condition is simply $\Ex\norm{X}^2>0$. The JiT implementation used in our diagnostics instead sets both the weights and bias of \texttt{final\_layer.linear} to zero. This layer predicts a patch vector, and its trainable bias is shared across patches. Writing an image as $P$ target patches $X_1,\ldots,X_P$,
\begin{equation}
S^\top X=\sum_{j=1}^P X_j,\qquad
G_{\theta_0}(t)\ge4\Ex\norm{\sum_{j=1}^P X_j}^2.
\label{eq:jit_patch_bias}
\end{equation}
The sufficient nondegeneracy condition is therefore a positive second moment of the summed target patches. Zero output does not imply zero parameter gradient: the output bias remains trainable at initialization.

\paragraph{Zero-initialized identity residual.}
If $\hat X_{\theta_0}(Z_t,t)=Z_t$, then $r_0(t)=(1-t)(E-X)$. Independence and the zero mean of $E$ imply
\begin{equation}
\Ex\norm{S^\top(E-X)}^2
=\norm{S}_{\mathrm F}^2+\Ex\norm{S^\top X}^2>0.
\end{equation}
Consequently,
\begin{equation}
G_{\theta_0}(t)\ge
4(1-t)^2\big(\norm{S}_{\mathrm F}^2+\Ex\norm{S^\top X}^2\big),
\end{equation}
which supplies Regime B. This argument concerns a trainable residual output bias; the identity skip alone does not supply a parameter gradient.

\paragraph{Nonzero-initialized direct output.}
More generally, suppose $r_0(t)\to r_0(1)$ in $L^2$ as $t\uparrow1$ and let $d_0=\Ex\norm{S^\top r_0(1)}^2>0$. Continuity of the $L^2$ norm gives $\Ex\norm{S^\top r_0(t)}^2\ge d_0/2$ on a sufficiently small endpoint neighborhood, yielding Regime A with $c_J=d_0/2$.

The Binary-MNIST U-Net uses a direct convolutional output \texttt{out\_conv} with one trainable scalar bias shared across its $D$ pixels. It is not zero-initialized. Conditional on all other initial parameters, write its endpoint residual as $r_0(1)=h(X,C)-X+b\mathbf 1$, with dropout randomness suppressed in the notation for $h$. The squared projected residual is
\begin{equation}
\Ex\big[\big(\mathbf 1^\top(h(X,C)-X)+Db\big)^2\big].
\end{equation}
As a function of $b$, this is a nonnegative quadratic with positive leading coefficient $D^2$, so it vanishes at at most one value of $b$. A continuous initialization of $b$ independent of the remaining parameters therefore gives $d_0>0$ almost surely. The implemented GroupNorm, convolutions, and SiLU activations are continuous for fixed dropout masks; the final GroupNorm followed by SiLU and a finite convolution bounds the initialized output for fixed parameters. With finite data second moments, dominated convergence gives the required $L^2$ endpoint continuity, also after averaging over dropout masks. This establishes Regime A under the stated initialization model without global Jacobian nondegeneracy.

\subsection{Proof of Theorem~\ref{thm:init_divergence}}
\label{app:proof_divergence}

\begin{proof}
By Equations~\eqref{eq:init_orders} and~\eqref{eq:optimizer_second_moments}, for every $b\in(t_0,1)$,
\begin{equation}
M_v(q)
\ge
\begin{cases}
4c_J\int_{t_0}^{b}q(t)(1-t)^{-4}dt, & \text{Regime A},\\
4c_J\int_{t_0}^{b}q(t)(1-t)^{-2}dt, & \text{Regime B}.
\end{cases}
\end{equation}
Now assume Equation~\eqref{eq:sampler_endpoint_lower_bound}. For every $b\in(t_1,1)$, the nonnegativity of the integrand gives
\begin{equation}
M_v(q)
\ge4c_Jc_q\int_{t_1}^{b}(1-t)^{\beta-p}dt.
\end{equation}
If $\beta\le p-1$, the truncated integral tends to infinity as $b\uparrow1$. Monotone convergence therefore yields $M_v(q)=\infty$. Uniform sampling has $c_q=1$ and $\beta=0$, so the condition holds for both $p=4$ and $p=2$.
\end{proof}

At a trained state, Equation~\eqref{eq:optimizer_second_moments}, rather than the initialization orders in Equation~\eqref{eq:init_orders}, remains the exact integrability criterion.

\subsection{Proof of Proposition~\ref{prop:aligned_regularity}}
\label{app:alignment_proof}

\begin{proof}
Since $q$ is a probability density, Equation~\eqref{eq:optimizer_second_moments} gives
\begin{equation}
M_x(q)=\int_0^1q(t)G_\theta(t)dt
\le C_G(\theta)\int_0^1q(t)dt
=C_G(\theta).
\end{equation}
\end{proof}

A transparent sufficient condition for $C_G(\theta)<\infty$ is the following bound, with the essential supremum taken under Lebesgue time measure and the underlying data--noise law:
\begin{equation}
K_\theta:=\operatorname*{ess\,sup}_{t\in(0,1),\,\omega\in\Omega}\norm{J_\theta(t;\omega)}_{\mathrm{op}}<\infty,
\qquad
C_R(\theta):=\sup_{t\in(0,1)}\Ex\norm{r_\theta(t)}^2<\infty,
\label{eq:jacobian_residual_sufficient}
\end{equation}
Indeed, for Lebesgue-almost every $t\in(0,1)$,
\begin{equation}
G_\theta(t)
\le4\Ex\left[\norm{J_\theta(t)}_{\mathrm{op}}^2\norm{r_\theta(t)}^2\right]
\le4K_\theta^2C_R(\theta).
\end{equation}
The residual-moment condition does not follow from the Jacobian bound and is stated separately. Neither condition requires a pointwise bounded output head. A weaker joint-moment alternative is
\begin{equation}
C_{Jr}(\theta):=\sup_{t\in(0,1)}\Ex\left[\norm{J_\theta(t)}_{\mathrm{op}}^2\norm{r_\theta(t)}^2\right]<\infty,
\end{equation}
which directly yields $C_G(\theta)\le4C_{Jr}(\theta)$. Separate uniform fourth-moment bounds on $\norm{J_\theta(t)}_{\mathrm{op}}$ and $\norm{r_\theta(t)}$ imply this condition by Cauchy--Schwarz.

\paragraph{Other aligned heads.}
For direct velocity prediction with residual $r_v=\hat V_\theta-(X-E)$ and Jacobian $J_\theta^v=\nabla_\theta\hat V_\theta$, the identical argument gives
\begin{equation}
\operatorname*{ess\,sup}_{t\in(0,1)}\Ex\norm{2(J_\theta^v)^\top r_v}^2
\le4(K_\theta^v)^2C_R^v(\theta)
\end{equation}
under the corresponding fixed-state Jacobian and residual-moment conditions. For summed BCE-with-logits, let $A_\theta$ be the logits and $Y\in\{0,1\}^D$. Its gradient is
\begin{equation}
g_{\mathrm{BCE}}=(J_\theta^A)^\top(\sigma(A_\theta)-Y).
\end{equation}
Because $\norm{\sigma(A_\theta)-Y}^2\le D$, a Jacobian bound $\norm{J_\theta^A}_{\mathrm{op}}\le K_\theta^A$ holding almost surely for Lebesgue-almost every $t$, with the same constant $K_\theta^A$, gives $\operatorname*{ess\,sup}_{t\in(0,1)}\Ex\norm{g_{\mathrm{BCE}}}^2\le(K_\theta^A)^2D$. Fixed class weights only multiply this bound by the squared maximum weight. These statements, like Proposition~\ref{prop:aligned_regularity}, are conditional fixed-state bounds rather than training-trajectory guarantees.

\subsection{Proof of Proposition~\ref{prop:lognormal_integrability}}
\label{app:lognormal}

\begin{proof}
Let $U\sim\mathcal N(m,s^2)$ with $s>0$ and $t=\sigma(U)$. For any nonnegative measurable $F$,
\begin{equation}
\int_0^1q_{\mathrm{LN}}(t)F(t)dt
=\Ex_{U\sim\mathcal N(m,s^2)}[F(\sigma(U))].
\end{equation}
Suppose the base signal-gradient second moment has finite polynomial endpoint growth: for some $C<\infty$ and $p\ge0$,
\begin{equation}
G_\theta(t)\le C(1-t)^{-p}
\end{equation}
on a terminal interval. Under velocity-space mismatch, the integrand is bounded by $C(1-t)^{-(p+4)}$. Since
\begin{equation}
(1-\sigma(U))^{-1}=1+e^U,
\end{equation}
the transformed integrand grows at most as a constant multiple of $e^{(p+4)U}$ for sufficiently large $U$, whose expectation is finite under a Gaussian distribution. Thus the terminal contribution $\int_{t_0}^1q_{\mathrm{LN}}(t)a(t)^2G_\theta(t)dt$ is finite for every finite polynomial endpoint order. This argument does not control $(0,t_0]$: the full second moment is finite only when that nonterminal contribution is also integrable. This conclusion requires an upper-growth condition; an initialization-time lower bound such as Assumption~\ref{ass:init_regimes} is not sufficient for it.

If $C_G(\theta)=\operatorname*{ess\,sup}_{t\in(0,1)}G_\theta(t)<\infty$, then the full integral satisfies
\begin{equation}
\int_0^1q_{\mathrm{LN}}(t)a(t)^2G_\theta(t)dt
\le C_G(\theta)\Ex\left[(1-\sigma(U))^{-4}\right]<\infty.
\end{equation}
The last expectation is finite because $(1-\sigma(U))^{-1}=1+e^U$ and a Gaussian has finite exponential moments of every finite order.
\end{proof}

\section{Experimental Details}
\label{app:experimental_details}

\subsection{Paired initialization-regime diagnostic}
\label{app:toy_initialization}

Table~\ref{tab:initialization_settings} summarizes the protocol for the direct measurements in Table~\ref{tab:toy_initialization}.

\begin{table}[H]
\centering
\small
\caption{Paired MLP initialization diagnostic: backbone and measurement settings.}
\label{tab:initialization_settings}
\begin{tabular}{p{0.28\linewidth} p{0.64\linewidth}}
\toprule
\textbf{Item} & \textbf{Setting} \\
\midrule
\multicolumn{2}{l}{\textit{Backbone and initialization}} \\
Branch & Gated MLP; two hidden layers, width 256; SiLU activations and sigmoid gates \\
Timestep embedding & 128-dimensional sinusoidal embedding \\
Hidden and gate layers & Xavier-uniform weights; zero biases \\
Final linear layer & Zero weights and bias; all parameters remained trainable \\
Predictors & A: $f_\theta(Z_t,t)$; B: $Z_t+f_\theta(Z_t,t)$ \\
\midrule
\multicolumn{2}{l}{\textit{Data and gradient measurement}} \\
Data / noise & 128 independent pairs $X,E\sim\mathcal N(0,I_{16})$, reused at all timesteps and seeds \\
Initialization seeds & 42, 43, 44; paired parameter tensors across A and B \\
Endpoint grid & 26 distinct distances $\epsilon=1-t$ from $10^{-1}$ to $10^{-6}$ \\
Loss / statistic & Mean-coordinate MSE; average of individual-example squared full-parameter gradient norms \\
Precision / updates & FP64; separate backward passes for each loss; no optimizer or parameter updates \\
Safeguards & No denominator floor, gradient clipping, or mixed precision \\
\bottomrule
\end{tabular}
\end{table}

For each seed, the two predictors used exactly the same parameter tensors and differed only by the addition of the parameter-free input $Z_t$ to the branch output. Their parameter Jacobians are consequently identical. The direct output was initially zero; the identity-residual output was initially $Z_t$. With the final layer zero-initialized, only its weight and bias gradients are active at this state. This is a controlled test of the two constructed initialization regimes, separate from the trained Gaussian sampler experiment below.

We evaluated 26 distinct endpoint distances from $10^{-1}$ to $10^{-6}$ in FP64. At every distance and for every example, we performed separate forward and backward passes for the two losses:
\begin{align}
\ell_{x,j}^{\mathrm{mean}}(t)&=\frac{1}{16}\norm{\hat X_{\theta_0}(Z_{t,j},t)-X_j}^2,\\
\ell_{v,j}^{\mathrm{mean}}(t)&=\frac{1}{16}\norm{\frac{\hat X_{\theta_0}(Z_{t,j},t)-Z_{t,j}}{1-t}-(X_j-E_j)}^2.
\end{align}
The recorded second moments are
\begin{equation}
\widehat G_s(t)=\frac{1}{128}\sum_{j=1}^{128}
\norm{\nabla_\theta\ell_{s,j}^{\mathrm{mean}}(t)}^2,\qquad s\in\{x,v\}.
\end{equation}
These are averages of individual-example squared gradient norms, not squared norms of batch-averaged gradients. Mean-coordinate MSE scales each moment by $16^{-2}$ relative to the summed-loss convention in the theory. The velocity-loss observations were obtained from their own backward passes, not by multiplying the aligned measurements by $(1-t)^{-4}$. No denominator floor, gradient clipping, or mixed precision was used. Parameter hashes remained unchanged throughout each diagnostic and matched between the paired predictors; data, noise, and timesteps also matched exactly.

As an implementation check, we compared the independently computed full gradient vectors with $g_v=(1-t)^{-2}g_x$ on the first example at every timestep and initialization, and checked the squared-norm ratio on all examples. Table~\ref{tab:toy_initialization} selects three distances from this grid and reports the mean and sample standard deviation of the three initialization-specific moments. This standard deviation describes initialization variability on the fixed data, not uncertainty over independent data draws.

Figure~\ref{fig:toy_initialization} displays all 128 individual squared-gradient observations separately for each seed at $\epsilon=1-t\approx10^{-1},10^{-3},10^{-6}$, without smoothing or sample-level jitter. For B, each recorded value was divided by the actual $\epsilon^2$; the normalization is prescribed by the initialization construction, not fitted to the observations. Boxes and whiskers summarize variation across the fixed examples within one initialization, not uncertainty across initialization seeds. The three seed groups and different timesteps reused the same data and must not be pooled as independent data draws. These finite-grid measurements check empirical nondegeneracy; they do not by themselves establish a uniform population lower bound or asymptotic divergence.

\begin{figure}[!htbp]
\centering
\includegraphics[width=0.90\linewidth,trim=5.2bp 7.1bp 5.2bp 9bp,clip]{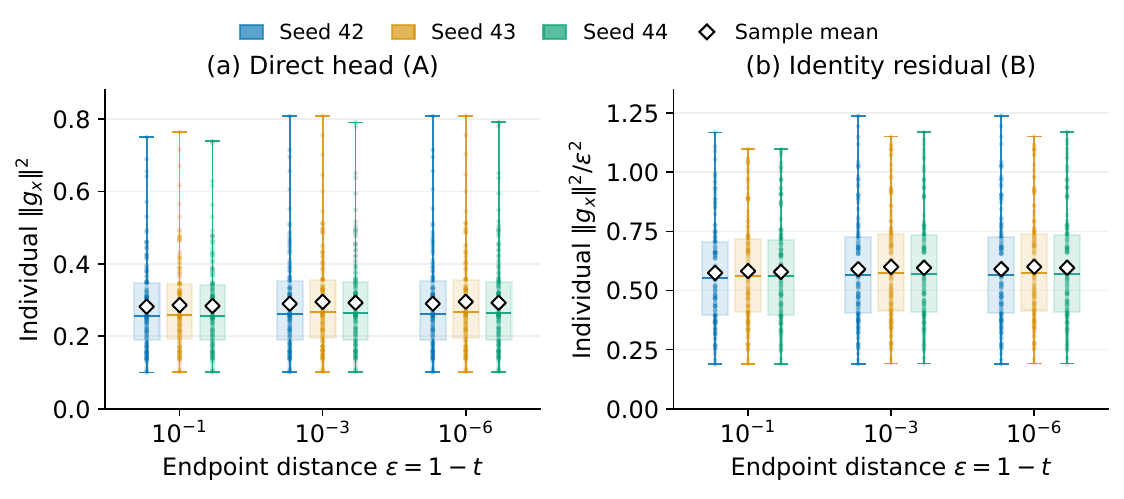}
\caption{Individual-example distributions for the initialization control. Each stack contains 128 measurements for one seed and endpoint distance. Left: squared aligned-gradient norms for A. Right: squared aligned-gradient norms divided by $\epsilon^2$ for B. Boxes show the interquartile range, whiskers the minimum and maximum, and diamonds the sample mean. Unlike Table~\ref{tab:toy_initialization}, the right panel is endpoint-normalized to display variation within each group.}
\label{fig:toy_initialization}
\end{figure}

\subsection{Controlled Gaussian protocol}
\label{app:toy_gaussian}

In Figure~\ref{fig:toy_gaussian}, we compared Uniform and Logit-Normal sampling over three paired initializations (seeds 42, 43, 44). Each pair used identical initial parameter tensors and shared data--noise draws at every update; only the timestep sampler changed. Table~\ref{tab:gaussian_settings} records the settings. This experiment measured the loss and gradient trajectories obtained when optimizing a fixed pointwise loss under different timestep samplers.

\begin{table}[!ht]
\centering
\small
\caption{Gaussian sampler-control training configuration.}
\label{tab:gaussian_settings}
\begin{tabular}{p{0.28\linewidth} p{0.64\linewidth}}
\toprule
\textbf{Item} & \textbf{Setting} \\
\midrule
\multicolumn{2}{l}{\textit{Backbone and data}} \\
Backbone & Gated MLP; two hidden layers of width 256, SiLU activations; 128-dimensional sinusoidal timestep embedding \\
Data / interpolation & $X,E\sim\mathcal N(0,I_{16})$ independently; $Z_t=tX+(1-t)E$ \\
Prediction / objective & Direct $x$-prediction, zero-initialized final layer, no identity skip; raw velocity MSE averaged over examples and coordinates \\
\midrule
\multicolumn{2}{l}{\textit{Training and logging}} \\
Optimizer / learning rate & Adam, $10^{-4}$; betas $(0.9,0.999)$, $\varepsilon_{\mathrm{Adam}}=10^{-8}$ \\
Budget / batch size & 5,000 updates / 1,000 \\
Timestep sampler & Uniform $(0,1)$ or $T=\sigma(U)$, $U\sim\mathcal N(-0.8,0.8^2)$ \\
Numerics / trace & FP64; no denominator floor or gradient clipping; pre-update full-parameter minibatch-gradient norm \\
\bottomrule
\end{tabular}
\end{table}

For each pair, a common uniform quantile $u$ gave $T_{\mathrm U}=u$ and $T_{\mathrm{LN}}=\sigma(-0.8+0.8\Phi^{-1}(u))$, where $\Phi$ is the standard normal CDF. We computed the loss directly as the mean squared value of $(\hat X_\theta-Z_T)/(1-T)-(X-E)$, without a denominator safeguard. The zero-initialized direct head reused the parameter snapshots from the initialization diagnostic; the sampler comparison changed neither the head nor the pointwise loss within each pair.

All six runs completed 5,000 updates with finite losses and gradients. The three Uniform runs had median gradient norms of order $10^3$ and maxima ranging from $10^9$ to $10^{12}$, whereas Logit-Normal medians were of order $10^{-1}$ and maxima of order $10^0$. Figure~\ref{fig:toy_gaussian}(a) shows the exact bin-averaged sampler densities, computed from their CDFs rather than empirical timestep counts. Panels (b,c) retain all per-update traces as faint lines and overlay the arithmetic mean across seeds followed by a trailing 25-update moving average. The loss panel uses a linear left axis for Logit-Normal and a logarithmic right axis for Uniform, as indicated by the sampler-colored labels. The native losses reflect different timestep-averaged objectives, and are not a matched-distribution prediction-quality comparison. The logged minibatch-gradient norm is distinct from the individual-example second moment estimated in Table~\ref{tab:toy_initialization}.

The experiment archive contains the model source, initial states, configuration and source hashes, input-stream checks, final parameters, and every update's native loss, signal MSE, gradient norm, and minimum endpoint distance. These records make the pairing and absence of denominator clipping auditable.

\FloatBarrier
\subsection{Binary-MNIST protocol}
\label{app:bmnist_protocol}

Table~\ref{tab:bmnist_settings} groups the backbone, training, and evaluation settings used for Figure~\ref{fig:bmnist_corrected}. All configurations shared the dataset, split, generation code, and FID evaluator.

\begin{table}[!ht]
\centering
\small
\caption{Binary-MNIST objective--sampler matrix: shared experimental configuration.}
\label{tab:bmnist_settings}
\begin{tabular}{p{0.28\linewidth} p{0.64\linewidth}}
\toprule
\textbf{Item} & \textbf{Setting} \\
\midrule
\multicolumn{2}{l}{\textit{Backbone and data}} \\
Backbone & Class-conditional U-Net \citep{ronneberger2015unet} \\
Conditioning & $Z_t$, timestep through a two-layer SiLU MLP, and learned digit-label embedding \\
Data & Deterministically thresholded MNIST \citep{lecunmnist}, mapped to $\{-1,+1\}$; fixed train/validation split \\
\midrule
\multicolumn{2}{l}{\textit{Training: flow matching}} \\
Interpolation & $Z_t=tX+(1-t)E$, $E\sim\mathcal N(0,I)$ \\
Optimizer / learning rate & Adam \citep{kingma2015adam}, $10^{-4}$ \\
Epochs / batch size & 100 / 128 \\
Training seeds & 42, 43, 44 \\
Timestep sampling & Uniform $(0,1)$ or $\operatorname{logit}(T)\sim\mathcal N(-0.8,0.8^2)$ \\
Prediction--loss pairings & $x$-to-$v$ MSE, $x$-MSE, $v$-MSE, and BCE-with-logits \\
\midrule
\multicolumn{2}{l}{\textit{Sampling and evaluation}} \\
ODE sampler & Euler, 50 steps; no hard thresholding during integration \\
BCE signal mapping & $\hat X_\theta=2\sigma(A_\theta)-1$ before velocity computation \\
FID samples & 50,000 conditional samples, balanced across digits \\
FID features & 128-dimensional penultimate features of a converged MNIST classifier \\
Classifier training & Adam, learning rate $10^{-3}$, batch size 128, 25 epochs; seed 314159 \\
Classifier selection & Highest validation accuracy; earliest epoch in a tie (epoch 20) \\
Generator selection & Final epoch-100 checkpoint for every configuration \\
Shared evaluation & Same classifier, real-feature statistics, generation seed, and integration code \\
Reported variability & Mean and sample standard deviation across training seeds \\
\bottomrule
\end{tabular}
\end{table}

The four configurations in Figure~\ref{fig:bmnist_corrected} used:
\begin{align}
v\text{-MSE: }&\quad \norm{\hat V_\theta-(X-E)}^2,\\
x\text{-to-}v\text{ MSE: }&\quad (1-t)^{-2}\norm{\hat X_\theta-X}^2,\\
x\text{-MSE: }&\quad \norm{\hat X_\theta-X}^2,\\
\mathrm{BCE: }&\quad \mathrm{BCEWithLogits}(A_\theta,(X+1)/2).
\end{align}
For BCE inference, logits were mapped back to the analog signal as $\hat X_\theta=2\sigma(A_\theta)-1$ before computing the implied velocity. All $x$-prediction models used the same Euler-50 ODE rule without hard thresholding during integration.

\paragraph{Classifier training and checkpoint selection.}
MNIST's 60,000 training images were thresholded at 0.5, mapped to $\{-1,+1\}$, and split into 54,000 training and 6,000 validation images with split seed 42. The \texttt{SimpleClassifier} contained two $3\times3$ convolutional layers (32 and 64 channels), each followed by ReLU and $2\times2$ max pooling, a 128-unit fully connected ReLU layer, and a ten-class output layer; dropout with probability 0.5 preceded the output layer during training. It was trained with cross-entropy and Adam for 25 epochs using learning rate $10^{-3}$, batch size 128, and seed 314159. We selected the checkpoint with the highest held-out validation accuracy, retaining the earlier checkpoint in a tie, rather than selecting by generated-sample FID. The recorded maximum occurred at epoch 20 (99.03\% validation accuracy). This checkpoint was frozen in evaluation mode, disabling dropout, for all feature extraction.

\paragraph{FID computation and generation protocol.}
FID-50K \citep{heusel2017gans} used the 128-dimensional activations after the \texttt{fc1} ReLU, not classification logits or probabilities. Real-reference features were extracted from 50,000 images in the training loader and cached once. For each generator, we produced 50,000 conditional samples (5,000 per digit) using generation seed 1000 and Euler-50 integration from Gaussian noise. The generator and classifier were both in evaluation mode. The implemented $x$-to-velocity denominator was $1-t+10^{-8}$ at $t=k/50$, $k=0,\ldots,49$; this numerical offset applied only to generation, not the raw training loss. Generated signals were passed directly to the classifier without thresholding or clipping. Only the displayed sample grids were mapped to $[0,1]$ and clipped for visualization. FID was computed from feature means and sample covariances in double precision. Every configuration shared the frozen classifier, cached real features, generation seed, and integration code.

\paragraph{Generator checkpoints and reported variability.}
Classifier selection was separate from generator selection: all entries in Figure~\ref{fig:bmnist_corrected} used the final epoch-100 generator checkpoint, not the checkpoint with the best validation loss or FID. Means and sample standard deviations summarize three independently trained generators (seeds 42--44), each evaluated with the same generation seed.

\paragraph{Evaluator diagnostics and per-seed results.}
Figure~\ref{fig:bmnist_classifier_diagnostic} reports the recorded classifier losses and validation accuracy throughout training. Validation accuracy increased from 97.60\% at epoch 1 to its maximum of 99.03\% at epoch 20, while the lowest validation cross-entropy occurred at epoch 7. The checkpoint was selected by accuracy, not loss. These curves document evaluator training and checkpoint selection. FID-ranking sensitivity to the choice of classifier checkpoint was not evaluated in this protocol.

\begin{figure}[!htbp]
\centering
\includegraphics[width=0.96\linewidth]{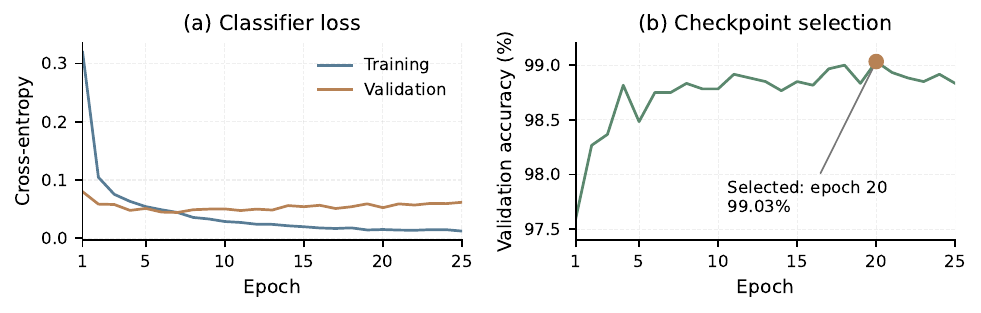}
\caption{BMNIST FID evaluator diagnostics from the classifier used for the reported experiments. Left: training and validation cross-entropy. Right: validation accuracy, with the selected epoch-20 checkpoint marked. Curves show the recorded epoch values without smoothing.}
\label{fig:bmnist_classifier_diagnostic}
\end{figure}

Table~\ref{tab:bmnist_per_seed} gives each generator's FID under the same frozen evaluator. Uniform-sampled mismatch had FID above $10^3$ in all three runs, whereas Logit-Normal-sampled mismatch and the aligned configurations remained at the order of $10^1$. The detailed values expose training-seed variability rather than uncertainty from repeated sampling of one generator.

\begin{table}[!htbp]
\centering
\small
\caption{BMNIST FID-50K for each training seed. All entries used the final epoch-100 generator, the selected epoch-20 classifier, and generation seed 1000. The last column reports the mean and sample standard deviation.}
\label{tab:bmnist_per_seed}
\begin{tabular}{llrrrr}
\toprule
Prediction--loss & Sampler & Seed 42 & Seed 43 & Seed 44 & Mean $\pm$ s.d. \\
\midrule
$x$-pred + $v$-MSE & Uniform & 1170.35 & 1524.46 & 1113.80 & $1269.53\pm222.57$ \\
$x$-pred + $v$-MSE & Logit-Normal & 10.54 & 13.02 & 9.52 & $11.02\pm1.80$ \\
$x$-pred + $x$-MSE & Uniform & 15.39 & 9.67 & 13.60 & $12.88\pm2.93$ \\
$x$-pred + $x$-MSE & Logit-Normal & 8.45 & 13.38 & 17.02 & $12.95\pm4.30$ \\
$v$-pred + $v$-MSE & Uniform & 24.24 & 16.32 & 18.46 & $19.67\pm4.10$ \\
$v$-pred + $v$-MSE & Logit-Normal & 16.81 & 15.58 & 10.34 & $14.24\pm3.44$ \\
$x$-pred + BCE & Uniform & 12.21 & 8.67 & 13.16 & $11.35\pm2.36$ \\
$x$-pred + BCE & Logit-Normal & 7.61 & 10.47 & 8.93 & $9.01\pm1.43$ \\
\bottomrule
\end{tabular}
\end{table}

\FloatBarrier
\subsection{Binary-MNIST implementation checks}
\label{app:bmnist_audit}

All Binary-MNIST configurations used the same threshold-binarized dataset, train/validation split, generation code, and FID evaluator. For BCE, logits were mapped to the analog signal as $2\sigma(A_\theta)-1$ before computing the ODE velocity. Validation mode was confined to evaluation, best and final checkpoints were stored separately, and every reported seed corresponds to an independently trained model. These checks ensure that the comparison isolates the stated prediction--loss pairing and timestep sampler.

\FloatBarrier
\subsection{Normalized endpoint-cap sweep}
\label{app:cap_sweep}

The sweep used the conditional U-Net backbone and Binary-MNIST evaluation protocol in Appendix~\ref{app:bmnist_protocol}. Each run used 100 epochs, batch size 128, learning rate $10^{-4}$, and the fixed split with split seed 42. We crossed $C\in\{10^2,10^4,10^6,10^8\}$ with training seeds 42, 43, and 44 under Uniform sampling on the full timestep interval. The loss is per-example signal MSE multiplied by
\begin{equation}
\widetilde w_C(t)=\frac{\min\{(1-t)^{-2},C\}}{2\sqrt C-1}.
\label{eq:cap_weight}
\end{equation}
Evaluation used the final epoch-100 checkpoint, 50,000 generated samples, Euler-50 integration, and generation seed 1000. Table~\ref{tab:bmnist_cap} reports variability across the three training seeds.

For $C\ge1$, define $W_C(t)=\min\{(1-t)^{-2},C\}$. Under Uniform $t$,
\begin{align}
\Ex W_C&=2\sqrt C-1,\\
\Ex W_C^2&=\frac{4C^{3/2}-1}{3}.
\end{align}
The normalized weight $\widetilde w_C=W_C/(2\sqrt C-1)$ has unit mean and
\begin{equation}
\Ex\widetilde w_C^2
=\frac{4C^{3/2}-1}{3(2\sqrt C-1)^2}
\sim\frac{\sqrt C}{3}.
\end{equation}
The corresponding FID-50K results are summarized in Table~\ref{tab:bmnist_cap} of Section~\ref{sec:empirical_amplification}.

The normalization fixes the mean analytical scale but changes the time profile of supervision. The sweep therefore evaluates the joint training effect of changing endpoint concentration and the time profile of supervision.

\FloatBarrier
\subsection{JiT-B/4 continuous-image protocol}
\label{app:jit_details}

For context, JiT explicitly adopts Logit-Normal timestep sampling following the Stable Diffusion 3 rectified-flow work \citep{esser2024scaling,li2025back}, and ELF uses the same sampler family \citep{hu2026elf}. JiT documents a $0.05$ floor on the denominator in the $x$-to-$v$ conversion \citep{li2025back}. ELF's paper does not state this clipping value, but its official implementation sets \texttt{t\_eps=0.05} and replaces $1-t$ by $\max(1-t,\texttt{t\_eps})$ in the conversion.\footnote{See lines 110--120 of the \href{https://github.com/lillian039/ELF/blob/main/src/utils/sampling_utils.py\#L110-L120}{official ELF implementation}.} We record these implementation details to distinguish the published safeguarded recipes from our near-unclipped and raw diagnostics.

Table~\ref{tab:jit_settings} summarizes the full objective--sampler matrix. In both mismatched training cells, the target and prediction were converted using the same $\max(1-t,10^{-6})$ denominator. This historical near-unclipped implementation caps the loss weight at $10^{12}$ and is much less protective than the official JiT floor $0.05$; the matrix is therefore a controlled JiT-backbone evaluation rather than an exact reproduction of the published recipe. The generation safeguard is recorded separately from the training floor.

\begin{table}[!ht]
\centering
\small
\caption{JiT-B/4 objective--sampler matrix: training and generation settings.}
\label{tab:jit_settings}
\begin{tabular}{p{0.28\linewidth} p{0.64\linewidth}}
\toprule
\textbf{Item} & \textbf{Setting} \\
\midrule
\multicolumn{2}{l}{\textit{Backbone and data}} \\
Model / dataset & JiT-B/4 \citep{li2025back}; Tiny-ImageNet \citep{le2015tinyimagenet}, $64\times64$ \\
Pairings / repeats & $x$-to-$v$ MSE, $x$-MSE, $v$-MSE; one run per objective--sampler cell \\
\midrule
\multicolumn{2}{l}{\textit{Training: flow matching}} \\
Epochs / effective batch & 400 / 256 \\
Optimizer & AdamW \citep{loshchilov2017adamw}, weight decay 0.03 \\
Learning-rate schedule & Base $5\times10^{-5}$; cosine decay with 5,000-step warmup \\
Timestep sampling & Uniform $(0,1)$ or $\operatorname{logit}(T)\sim\mathcal N(-0.8,0.8^2)$ \\
Mismatch conversion & Denominator $\max(1-t,10^{-6})$ for prediction and target in both sampler cells \\
\midrule
\multicolumn{2}{l}{\textit{Sampling and evaluation}} \\
ODE sampler / domain & Heun-50: 49 Heun updates followed by one Euler update; $t=0$ to $t=1$ \\
Generation safeguard & $\texttt{t\_eps}=0.05$ for all $x$-prediction cells \\
Classifier-free guidance & Scale 3.0 \citep{ho2022classifierfree}, active on $(0.1,1.0)$; this is not the ODE domain \\
FID-50K & 50,000 samples; \texttt{torch-fidelity} \citep{obukhov2020torchfidelity}, Inception-V3 pool3 features \\
\bottomrule
\end{tabular}
\end{table}

\begin{figure}[p]
\centering
\includegraphics[width=0.96\linewidth,height=0.82\textheight,keepaspectratio,trim=6.6bp 4.9bp 5.2bp 5.6bp,clip]{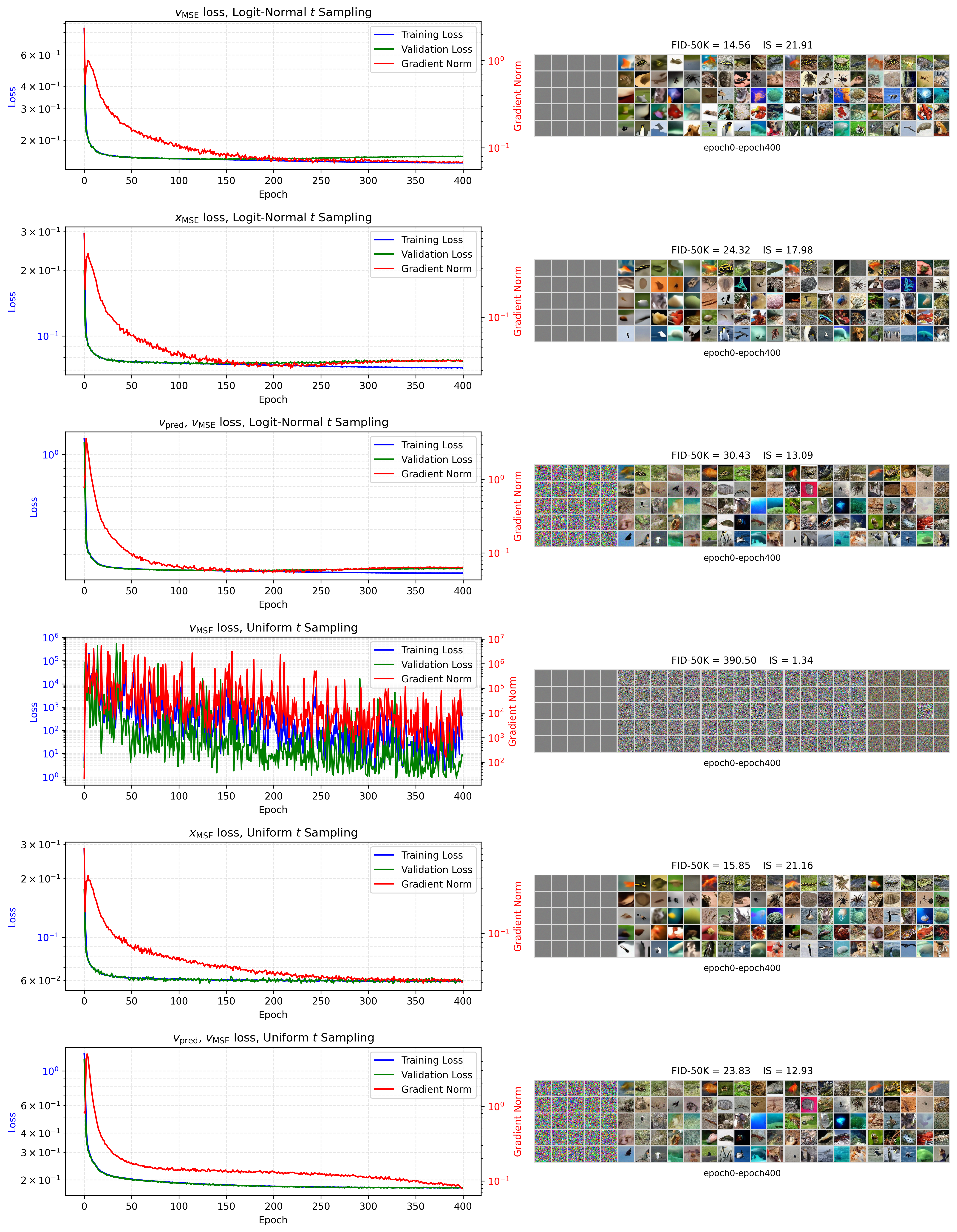}
\caption{Expanded JiT-B/4 Tiny-ImageNet diagnostics for all six objective--sampler configurations in Figure~\ref{fig:jit_dynamics}. This archived view retains the complete epoch-0-to-400 sample evolution together with the original per-run loss and gradient trajectories; the main-text figure shows only the final epoch-400 samples.}
\label{fig:jit_samples}
\end{figure}

\paragraph{Paired early-training protocol.}
\label{app:jit_early}
We compared three JiT-B/4 branches under Uniform sampling: aligned signal MSE, raw mismatched velocity MSE, and velocity MSE with a denominator floor. All three started from the same random initialization and preserved data order, augmentations, labels, class-dropout draws, timesteps, interpolation noise, and learning-rate schedule. Table~\ref{tab:jit_early_settings} records the shared settings and the sole objective/conversion difference. Each record reports the norm of the accumulated minibatch gradient before the optimizer update; loss reductions averaged over output coordinates. Figure~\ref{fig:jit_smoke_full} shows the complete three-branch diagnostic, including a common signal-prediction MSE for comparing fitting across objectives.

\begin{table}[!ht]
\centering
\small
\caption{Paired JiT-B/4 early-training diagnostic configuration.}
\label{tab:jit_early_settings}
\begin{tabular}{p{0.28\linewidth} p{0.64\linewidth}}
\toprule
\textbf{Item} & \textbf{Setting} \\
\midrule
Backbone / data & JiT-B/4; Tiny-ImageNet, $64\times64$ \\
Sampler / initialization & Uniform $(0,1)$; common initialization seed 0 \\
Compared branches & Signal MSE; velocity MSE with exact denominator $1-t$; velocity MSE with floor $10^{-6}$ \\
Budget / batch & 5,000 updates; four microbatches of 64, effective batch 256 \\
Learning-rate schedule & Base $5\times10^{-5}$; 5,000-step warmup on a shared 156,400-step schedule \\
Weight decay / label dropout & 0.03 / 0.1 \\
Precision & BF16 autocast; FP32 loss accumulation \\
Gradient recording & Pre-update accumulated minibatch-gradient norm; no gradient clipping or timestep truncation \\
Loss reduction & Mean over examples and output coordinates \\
\bottomrule
\end{tabular}
\end{table}

\begin{figure}[!htbp]
\centering
\includegraphics[width=\linewidth,trim=4.7bp 5.2bp 7.6bp 9.5bp,clip]{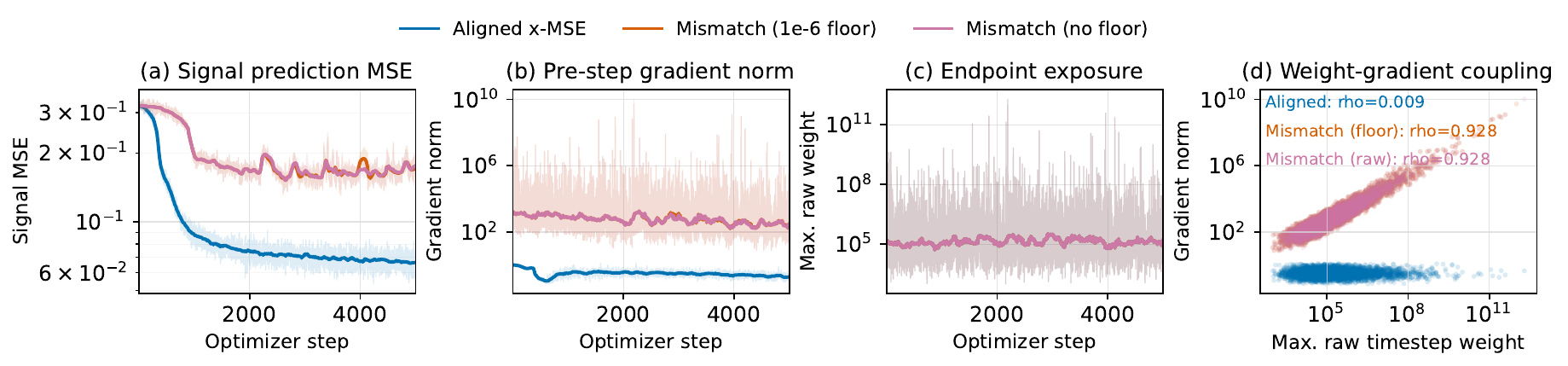}
\caption{Supplementary paired JiT-B/4 diagnostic under Uniform sampling. Three branches shared initialization and stochastic inputs, using aligned signal MSE, raw mismatched velocity MSE, or velocity MSE with a $10^{-6}$ denominator floor. Panels show signal MSE, pre-update minibatch-gradient norms, endpoint exposure, and weight--gradient associations. Thin curves retain all observations; thick curves show 101-step rolling medians.}
\label{fig:jit_smoke_full}
\end{figure}

The median pre-update gradient norm was $0.306$ for alignment and $590.98$ for raw mismatch. The supplementary $10^{-6}$ floor activated at step 2,192, where the largest sampled timestep was $0.9999992847$. The largest gradient norms were $5.44\times10^9$ with the floor and $1.06\times10^{10}$ without it; their medians were $596.91$ and $590.98$, respectively. Both mismatched branches retained higher early signal MSE than alignment. Thus the floor changed the extreme event, while the qualitative training contrast persisted. Correlations in Figure~\ref{fig:jit_smoke_full} use all 5,000 paired records and the maximum applied mismatch weight in each minibatch; the aligned branch is compared with the corresponding raw endpoint exposure. These associations describe the co-variation of sampled endpoint exposure and raw gradients along the recorded training trajectories.
All three branches remained numerically finite over the 5,000-step horizon.

\paragraph{Frozen-checkpoint protocol.}
\label{app:frozen_endpoint}
For Figure~\ref{fig:jit_endpoint}, we evaluated the raw, non-EMA weights from two epoch-399 checkpoints (400 completed epochs): aligned $x$-MSE with Uniform training and mismatched $x$-prediction+$v$-MSE with Logit-Normal training. Both evaluations used the same 32 fixed validation images from distinct, evenly spaced class indices, the same noise realizations, and 15 timesteps from $t=0$ to approximately $1-10^{-6}$. No model parameters were updated. Forward passes, losses, and gradient evaluations used FP32 without a denominator floor or gradient clipping.

\begin{figure}[!htbp]
\centering
\includegraphics[width=0.78\linewidth,trim=5.2bp 8bp 4.7bp 8bp,clip]{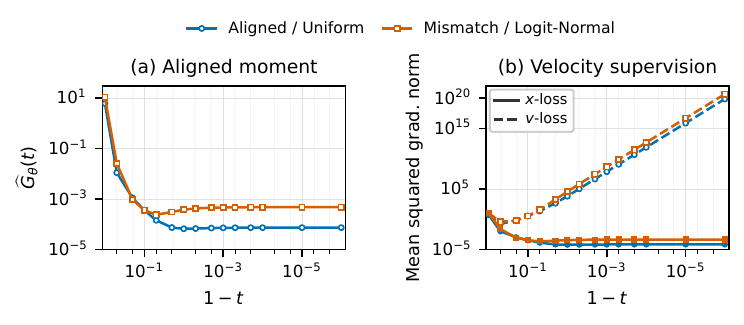}
\caption{Frozen JiT endpoint diagnostic after training. Left: empirical aligned-gradient second moment. Right: aligned (solid) and mismatched (dashed) losses at the same checkpoints. Colors identify the training recipe. Each point averages squared full-parameter gradient norms over the same 32 fixed validation image/noise pairs, evaluated one example at a time in FP32 without updates or clipping.}
\label{fig:jit_endpoint}
\end{figure}

At each timestep, we backpropagated one image at a time and averaged the squared Euclidean norms of the resulting full-parameter gradients. With $\ell_x^{\mathrm{mean}}=\ell_x/D$, the plotted empirical aligned moment is
\begin{equation}
\widehat G_\theta(t)=\frac{1}{n}\sum_{j=1}^n
\norm{\nabla_\theta\ell_{x,j}^{\mathrm{mean}}(t)}^2,
\qquad n=32.
\end{equation}
It is a fixed-set empirical counterpart of $G_\theta(t)/D^2$, rather than the squared norm of an averaged gradient. The coordinate normalization changes the magnitude but not the endpoint order. Because image selection is fixed and class-stratified, these values describe the diagnostic set rather than an unbiased estimate over the full training distribution.

For each checkpoint and input, the mismatched objective was also evaluated in the algebraically simplified form $(1-t)^{-2}\ell_x^{\mathrm{mean}}$, with the same model output. The resulting squared-gradient ratio is prescribed by the conversion identity. Comparing this form with direct velocity division gave maximum relative squared-gradient discrepancies below $2\times10^{-7}$ over the recorded grid; the ratio audit used the actual FP32 value of $1-t$. The empirical information is the behavior of the base moment itself: over the last three endpoint locations it approached a plateau in both checkpoints. The observed plateau concerns the recorded finite grid. The two checkpoints also differed in their training samplers, so the comparison reflects both objective and sampling choices.

For the aligned-trained checkpoint, reducing $1-t$ from approximately $10^{-4}$ to $10^{-6}$ left the aligned moment near $7.47\times10^{-5}$, while the mismatched moment rose from $7.47\times10^{11}$ to $7.10\times10^{19}$. The observed plateau is empirical; the amplification ratio follows from Equation~\eqref{eq:formal_conversion}. Together, these measurements show substantial endpoint amplification at the evaluated trained states over the recorded timestep grid.

\paragraph{Exploratory trained-state checks.}
We additionally recorded two paired restarts from trained checkpoints. Switching three trained Binary-MNIST checkpoints from aligned to mismatched supervision increased the first shared-minibatch gradient norm by $306\times$--$2766\times$; after 2,000 paired updates, FID-50K was $30.08\pm19.13$ for alignment and $1554.13\pm612.71$ for mismatch. A fresh-AdamW restart from a trained JiT checkpoint produced a $19{,}896\times$ first-minibatch raw-gradient ratio; the mismatch branch had 43\% higher mean signal MSE over 1,000 updates while remaining numerically finite. These finite-horizon observations describe changes in raw gradients, signal fitting, and generation following an objective switch. The interventions changed both the time weighting of the objective and the gradients presented to the optimizer.

\subsection{MIMO detection as a supplementary application}
\label{app:mimo}

We used the conditional flow-matching protocol and SGDiT backbone introduced by \citet{jiang2026sgdit}. This application provides supplementary evidence for the mismatch mechanism in conditional inference and additionally checks checkpoint-selection sensitivity. We consider the real-valued equivalent system
\begin{equation}
Y=HX+N,\qquad X\in\{-1,+1\}^{2N},\qquad N\sim\mathcal N(0,\sigma^2I),
\end{equation}
obtained from QPSK transmission over an i.i.d. Rayleigh-fading channel. The backbone combines a Soft Graph Transformer \citep{hong2025soft} with DiT-style timestep conditioning \citep{peebles2023scalable}; it conditions on $Y$, $H$, and the timestep and transports an initial Gaussian sample toward the conditional signal distribution. In every $x$-prediction configuration, the backbone directly outputs $\hat X_\theta(Z_t,t,Y,H)$ and the detector computes the implied velocity $(\hat X_\theta-Z_t)/(1-t)$ only for integration. In direct $v$-prediction configurations, the backbone output is the velocity itself. All objectives used the same two-step Euler detector. Training and integration truncated time at $t_{\max}=0.99$.

\begin{figure}[t]
\centering
\includegraphics[width=0.98\linewidth,trim=0bp 0bp 36bp 0bp,clip]{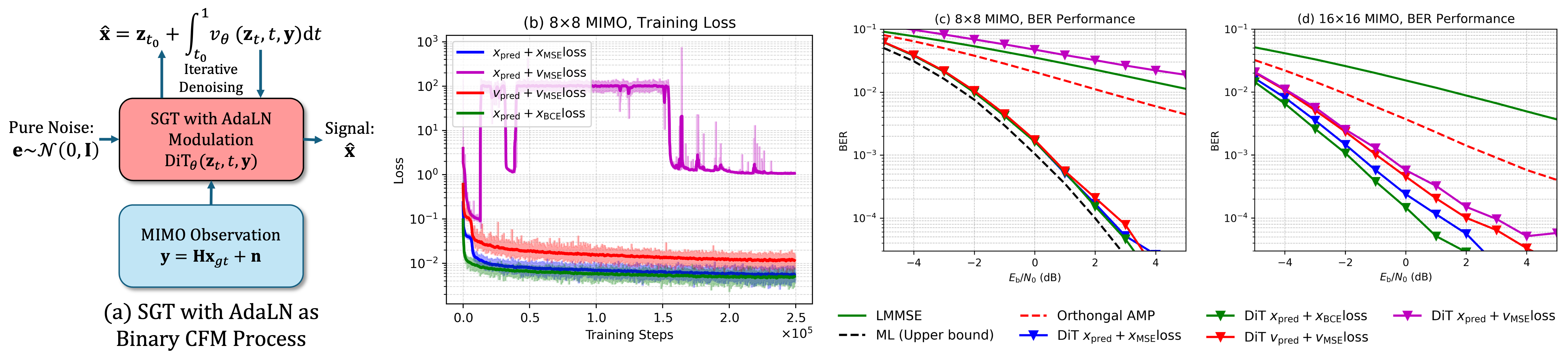}
\caption{Supplementary conditional MIMO detection. Aligned signal objectives avoided the severe instability of the tested $x$-prediction/velocity-loss mismatch and achieved lower BER under the reported recipes.}
\label{fig:mimo}
\end{figure}

The $8\times8$ system used learning rate $10^{-3}$ and the $16\times16$ system used $10^{-4}$. We evaluated both the final 250K checkpoint and the early 13K checkpoint favorable to the mismatched baseline. Aligned objectives retained a BER advantage under both checkpoint rules. The performance gaps persisted under the evaluated truncation and checkpoint settings.

\section{Additional Scope Clarifications}
\label{app:scope}

\paragraph{What ``robust'' means.}
We use robustness narrowly to mean reduced sensitivity of training dynamics to terminal timestep exposure and sampler choice. It does not refer to adversarial robustness, distribution-shift robustness, or a guarantee of convergence. Throughout this work, robustness is understood in this optimization setting.

\paragraph{Relation to discrete baselines.}
The binary experiments use an analog-bit representation within a continuous Gaussian path. This setting allows prediction--loss pairing and timestep sampling to be varied within the same continuous-state formulation. Categorical diffusion and discrete flow matching address the complementary question of selecting among different generative pipelines; the present comparisons focus on the internal objective--sampler mechanism.

\section{Compute and Assets}
\label{app:compute}

\Needspace{14\baselineskip}
\begin{wraptable}[12]{r}{0.5\textwidth}
\vspace{-0.5\baselineskip}
\centering
\caption{Published reference GPU specifications.}
\label{tab:gpu_specs}
\setlength{\tabcolsep}{4pt}
\begin{tabular}{lr}
\toprule
Specification & Reference value \\
\midrule
FP32 compute cores & 16,384 \\
Base clock (GHz) & 2.23 \\
Boost clock (GHz) & 2.52 \\
Standard memory & 24 GB GDDR6X \\
Memory interface & 384 bit \\
Graphics power (W) & 450 \\
\bottomrule
\end{tabular}
\end{wraptable}

Table~\ref{tab:gpu_specs} gives the published reference specifications of the GPU family used for the diagnostics. These nominal specifications describe the reference design, rather than host-specific memory configurations or measured training clocks.

Each diagnostic job used one GPU and executed its configurations serially; jobs ran on multiple hosts. The full JiT-B/4 matrix comprised six 400-epoch runs. One source run took approximately 26 hours, excluding queue time and some evaluation overhead. Each 5,000-step JiT smoke leg took about 48 minutes, and each 1,000-step paired restart leg took about 10 minutes. These are measured per-run times in the corresponding environments, not total exploratory compute.
\WFclear

MNIST \citep{lecunmnist} and Tiny-ImageNet \citep{le2015tinyimagenet} are established public research datasets. Binary MNIST is obtained by deterministic thresholding of MNIST. Tiny-ImageNet is derived from ImageNet and is used under its research and educational access terms. No human-subject data were collected, and no new dataset is released.

\end{document}